\documentclass[12 pt]{amsart}
\usepackage[letterpaper,margin=1.2in]{geometry}

\usepackage{microtype}
\usepackage{graphicx}
\usepackage{subfigure}
\usepackage{booktabs}
\usepackage[utf8]{inputenc}
\usepackage{amsmath,amsfonts,amsthm,amssymb,mathrsfs,dsfont,bbm} 
\usepackage{mathtools}
\usepackage{amsmath,amsfonts}
\usepackage{bbm}
\usepackage{hyperref}
\usepackage{enumerate}
\usepackage{xcolor}
\usepackage{caption}
\usepackage{algorithm}
\usepackage{algorithmic}
\usepackage{verbatim}

\newcommand{\bE}{\mathbb{E}}
\newcommand{\R}{\mathbb{R}}

\renewcommand{\vec}[1]{\mathbf{#1}}

\newcommand{\TODO}[1]{{\color{red}[TODO: #1]}}

\newcommand{\vx}{\vec{x}}

\newcommand{\vy}{\vec{y}}

\newcommand{\E}{\bE}      

\newcommand{\TV}{{\bf TV}}

\newtheorem{theorem}{Theorem}
\newtheorem{lemma}{Lemma}

\theoremstyle{definition}

\newtheorem{remark}{Remark}

\usepackage[foot]{amsaddr}
\usepackage{hyperref}

\begin{document}

\title{Multidimensional Scaling: Approximation and Complexity}
\author{Erik Demaine $^{1}$}
\address{$^1$Computer Science and Artificial Intelligence Laboratory, MIT, Cambridge, MA, USA}
\author{Adam Hesterberg $^{2}$}
\address{$^2$John A. Paulson School of Engineering and Applied Sciences, Harvard University, Cambridge, MA, USA}
\author{Frederic Koehler $^3$}
\address{$^3$Department of Mathematics, MIT, Cambridge, MA, USA}
\author{Jayson Lynch $^4$}
\address{$^4$Cheriton School of Computer Science, University of Waterloo, Waterloo, ON, Canada}
\author{John Urschel $^{3,\ast}$}
\address{$^{\ast}$Corresponding author}

\maketitle

\begin{abstract}
 Metric Multidimensional scaling (MDS) is a classical method for generating
    meaningful (non-linear) low-dimensional embeddings of high-dimensional data.
    MDS has a long history in the statistics, machine learning, and graph drawing
    communities. In particular, the Kamada-Kawai force-directed graph drawing method is equivalent to MDS and is one of the most popular ways in practice to embed graphs into low dimensions. Despite its ubiquity, our theoretical understanding of MDS remains limited as its objective function is highly non-convex. In this paper, we prove that minimizing the Kamada-Kawai objective is NP-hard and give a provable approximation algorithm for optimizing it, which in particular is a PTAS on low-diameter graphs. We supplement this result with experiments suggesting possible connections between our greedy approximation algorithm and gradient-based methods.
\end{abstract}

\section{Introduction}

Given the distances between data points living in a high dimensional space, how can we meaningfully visualize their relationships? This is a fundamental task in exploratory data analysis for which a variety of different approaches have been proposed. Many of these approaches seek to visualize high-dimensional data by embedding it into lower dimensional, e.g. two or three-dimensional, space. 

Metric multidimensional scaling (MDS or mMDS) \cite{kruskal1964multidimensional,kruskal1978multidimensional} is a classical approach to this problem which attempts to find a low-dimensional embedding that accurately represents the \emph{distances between points}. Originally motivated by applications in psychometrics, MDS has now been recognized as a fundamental tool for data analysis across a broad range of disciplines. See the texts \cite{kruskal1978multidimensional,borg2005modern} for more details, including a discussion of applications to data from scientific, economic, political, and other domains. Compared to other classical visualization tools like PCA\footnote{In the literature, PCA is sometimes referred to as \emph{classical multidimensional scaling}, in contrast to \emph{metric multidimensional scaling}, which we study in this work.}, metric multidimensional scaling has the advantage that it 1) is not restricted to linear projections of the data, i.e. it is nonlinear, and 2) is applicable to data from an arbitrary \emph{metric space}, rather than just Euclidean space. Because of this versatility, MDS has also become one of the most popular algorithms in the field of \emph{graph drawing}, where the goal is to visualize relationships between nodes (e.g. people in a social network). In this context, MDS was independently proposed by Kamada and Kawai \cite{kamada1989algorithm} as a \emph{force-directed graph drawing} method.

In this paper, we consider the algorithmic problem of computing the optimal embedding under the MDS/Kamada-Kawai objective. The Kamada-Kawai objective is to minimize the following energy/stress functional $E : \mathbb{R}^{rn} \to \mathbb{R}$ 
\begin{equation}\label{eqn:kk}
E(\vx_1,\ldots,\vx_n) = \sum_{i < j} \left(\frac{\|\vx_i - \vx_j\|}{d(i,j)} - 1\right)^2, 
\end{equation}
which corresponds to the physical situation where $\vx_1,\ldots,\vx_n \in \mathbb{R}^r$ are particles and for each $i \ne j$, particles $\vx_i$ and $\vx_j$ are connected by an idealized spring with equilibrium length $d(i,j)$ following Hooke's law with spring constant $k_{ij} = \frac{1}{d(i,j)^2}$. In applications to visualization, the choice of dimension is often small, i.e. $r = 1,2,3$. We also note that in \eqref{eqn:kk} the terms in the sum are sometimes re-weighted with vertex or edge weights, which we discuss in more detail later. 

In practice, the MDS/Kamada-Kawai objective \eqref{eqn:kk} is optimized via a heuristic procedure like gradient descent \cite{kruskal1964nonmetric,zheng2018graph} or stress majorization \cite{de1977applications,gansner2004graph}. Because the objective is non-convex, these algorithms may not reach the global minimum, but instead may terminate at approximate critical points of the objective function. Heuristics such as restarting an algorithm from different initializations and using modified step size schedules have been proposed to improve the quality of results. In practice, these heuristic methods do seem to work well for the Kamada-Kawai objective and are implemented in popular packages like \textsc{graphviz} \cite{ellson2001graphviz} and the \textsc{smacof} package in R.

\subsection{Our Results}
In this work, we revisit this problem from an approximation algorithms perspective. First, we resolve the computational complexity of minimizing \eqref{eqn:kk} by proving that finding the global minimum is NP-hard, even for graph metrics (where the metric is the shortest path distance on a graph).
Consider the decision version of stress minimization over graph metrics, which we formally define below: \\

 \noindent{\bf STRESS MINIMIZATION} \\
{\bf Input:} Graph $G = ([n],E)$, $r \in \mathbb{N}$, $L \ge 0$. \\
{\bf Output:} TRUE if there exists $\vx = (\vx_1,\ldots,\vx_n) \in \mathbb{R}^{nr}$ such that $E(\vx) \le L$;\\ $_{ }\qquad \qquad$ FALSE otherwise. \\

\begin{theorem}\label{thm:hardness-intro}
There exists a polynomial $p(n)$ such that the following gap version of {\bf STRESS MINIMIZATION} in dimension $r = 1$ is NP-hard: given an input graph $G$ with $n$ vertices and $L > 0$, return TRUE if there exists $\vx$ such that $E(\vx) \le L$ and return FALSE if for every $\vx$, $E(\vx) \ge L + 1/p(n)$. Furthermore, the problem is hard even restricted to input graphs with diameter bounded by an absolute constant.
\end{theorem}
\noindent
As a gap problem, the output is allowed to be arbitrary if neither case holds; the hardness of the gap formulation shows that there cannot exist a \emph{Fully-Polynomial Randomized Approximation Scheme} (FPRAS) for this problem if $\text{P} \ne \text{NP}$, i.e. the runtime cannot be polynomial in the desired approximation guarantee. Our reduction shows this problem is hard even when the input graph has \emph{low diameter} (even bounded by an absolute constant): this is a natural setting to consider since many real world graphs (for example, social networks \cite{dodds2003experimental}) and random graph models \cite{watts1998collective} indeed have low diameter due to the ``small-world phenomena''. Other key aspects of this hardness proof are: 1) we show the problem is hard even when the input $d$ is a graph metric, and 2) we show it is hard even in its canonical unweighted formulation \eqref{eqn:kk}. 

Given that computing the minimizer is NP-hard, a natural question is whether there exist polynomial time approximation algorithms for minimizing \eqref{eqn:kk}. We show that if the input graph has bounded diameter $D = O(1)$, then there indeed exists a \emph{Polynomial-Time Approximation Scheme} (PTAS) to minimize \eqref{eqn:kk}, i.e. for fixed $\epsilon > 0$ and fixed $D$ there exists an algorithm to approximate the global minimum of a $n$ vertex diameter $D$ graph up to multiplicative error $(1 + \epsilon)$ in time $f(\epsilon,D) \cdot poly(n)$. More generally, we show:
\begin{theorem}[Informal version of Theorem~\ref{thm:kk-alg}]\label{thm:kk-alg-intro}
Let $R > \epsilon > 0$ be arbitrary. Algorithm~\textsc{KKScheme} 
runs in time $n^2 (R/\epsilon)^{O(r R^4/\epsilon^2)}$ and outputs $\vx_1,\ldots,\vx_n \in \mathbb{R}^r$ with $\|\vx_i\| \le R$ such that
\[ \mathbb{E}\left[E(\vx_1,\ldots,\vx_n)\right] \le E(\vx^*_1,\ldots,\vx^*_n) + \epsilon n^2 \]
for any $\vx^*_1,\ldots,\vx^*_n$ with $\|\vx^*_i\| \le R$ for all $i$, where $\mathbb{E}$ is the expectation over the randomness of the algorithm.
\end{theorem}
\noindent where \textsc{KKScheme} is a simple greedy algorithm described in Section~\ref{sec:approx} below. The fact that this result is a PTAS for bounded diameter graphs follows from combining it with the two structural results regarding optimal Kamada-Kawai embeddings, which are of independent interest. The first (Lemma~\ref{lm:energy}) shows that the optimal objective value for low diameter graphs must be of order $\Omega(n^2)$ and the second (Lemma~\ref{lm:diam}) shows that the optimal KK embedding is ``contractive'' in the sense that the diameter of the output is never much larger than the diameter of the input.
\begin{lemma}[Informal version of Lemma~\ref{lm:energy}]
For any target dimension $r \ge 1$, all graphs of diameter $D = O(n^{1/r})$ satisfy
$E(\vx) = \Omega(n^2/D^r)$ for all $\vx$.
\end{lemma}
\begin{lemma}[Informal version of Lemma~\ref{lm:diam}]
For any graph of diameter $D$ and any target dimension $r \ge 1$, any global minimizer of $E(\vx)$ satisfies \[ \max_{i,j} \|\vx_i - \vx_j\| = O(D \log \log D), \]
i.e. the diameter of the embedding is $O(D \log \log D)$.
\end{lemma}

\subsection{Related Work}
\subsubsection{Other Approaches to Nonlinear Dimensionality Reduction and Visualization.} Recently, there has been renewed interest in force-directed graph layouts due to new applications in machine learning and data science. MDS itself is a popular technique for dimension reduction. Newer techniques, such as $t$-SNE \cite{maaten2008visualizing} and UMAP \cite{mcinnes2018umap}, can be viewed as similar type of force-directed weighted graph drawing with more complex objectives than Kamada-Kawai (see the discussion in \cite{mcinnes2018umap}); in comparison, some other dimensionality reduction methods, e.g. Laplacian eigenmaps \cite{belkin2003laplacian}, are based on spectral embeddings of graphs.

In practice, methods like t-SNE and UMAP appear to work quite well, even though they are based on optimizing non-convex objectives with gradient descent, which in general comes with no guarantee of success. Towards explaining this phenomena, $t$-SNE has been mathematically analyzed in a fairly specific setting where the data is split into well-separated clusters (e.g. generated by well-separated Gaussian mixtures); in this case, the works \cite{arora2018analysis,linderman2019clustering} prove that the visualization recovers the corresponding cluster structure. A difficulty when proving more general guarantees is that the $t$-SNE and UMAP objectives are fairly complex, and hence not so easy to mathematically analyze. 

Partially for this reason, in this work we focus on the simpler metric MDS/Kamada-Kawai objective. 
Experimentally, it has been observed that, using this objective, it is easy to find high quality minima in many different situations (see e.g. \cite{zheng2018graph}), but to our knowledge there has not been a mathematical explanation of this phenomena.

\subsubsection{Other related work.} In the multidimensional scaling literature, there has been some study of the \emph{local convergence} of algorithms like stress majorization, see for example \cite{de1988convergence}, which shows that stress majorization will converge quickly if in a sufficiently small neighborhood of a local minimum. This work seems to propose the first provable guarantees for global optimization. The closest previous hardness result is the work of \cite{cayton2006robust} where they showed that a similar problem is hard. In their problem: 1) the terms in \eqref{eqn:kk} are weighted by $d(i,j)$ and absolute value loss replaces the squared loss and 2) the input is an arbitrary pseudometric where nodes in the input are allowed to be at distance zero from each other. The second assumption makes the diameter (ratio of max to min distance in the input) infinite, and this is a major obstruction to modifying their approach to show Theorem~\ref{thm:hardness-intro}. See Remark~\ref{rmk:differences} for further discussion. A much earlier hardness result is the work of \cite{saxe1979embeddability}, in the easier (for proving hardness) case where distortion is only measured with respect to edges of the graph.

In the approximation algorithms literature, there has been a great deal of interest in optimizing the worst-case distortion of metric embeddings into various spaces, see e.g. \cite{badoiu2005approximation} for approximation algorithms for embeddings into one dimension, and \cite{deza2009geometry,naor2012introduction} for more general surveys of low distortion metric embeddings. Though conceptually related, the techniques used in this literature are not generally targeted for minimizing a measure of average pairwise distortion like \eqref{eqn:kk}.

In the graph drawing literature, there are a number of competing methods for drawing a graph, with the best approach depending on application \cite{battista1998graph}. Tutte's spring embedding theorem is often considered the seminal work in the force-directed layout community, and provides a method for producing a planar drawing of a three-connected planar graph \cite{tutte1963draw}. Though the problem under consideration in this work does indeed belong to the class of force-directed layouts, we stress the layouts under consideration do not minimize edge crossings in any sense.

\subsection{Notation.} In the remainder of the paper, we will generally assume the input is given as an unweighted graph to simplify notation; however, for the upper bounds (e.g. Theorem~\ref{thm:kk-alg-intro}) we do handle the general case of arbitrary metrics with distances in $[1,D]$ --- note that the lower bound of $1$ is without loss of generality after re-scaling. In the lower bound (i.e. Theorem~\ref{thm:hardness-intro}), we prove the (stronger) result that the problem is hard when restricted to graph metrics, instead of just for arbitrary metrics. We use standard asymptotic notation: $f(n) = O(g(n))$ means that $\lim\sup_{n \to \infty} \frac{f(n)}{g(n)} < \infty$, $f(n) = \Omega(g(n))$ means that $\lim \inf_{n \to \infty} \frac{f(n)}{g(n)} > 0$, and $f(n) = \Theta(g(n))$ means that $f(n) = \Omega(g(n))$ and $f(n) = O(g(n))$. The notation $[n]$ denotes $\{1,\ldots,n\}$. Unless otherwise noted, $\|\cdot\|$ denotes the Euclidean norm.

We also recall that a \emph{metric} $d : V \times V \to \mathbb{R}_{\ge 0}$ on a set $V$ is formally defined to be any function satisfying 1) $d(v,w) = 0$ iff $v = w$, 2) $d(v,w) = d(w,v)$ for all $v,w \in V$ and 3) $d(v,w) \le d(v,u) + d(u,w)$ for any $u,v,w \in V$. A \emph{pseudometric} relaxes 1) to the requirement that $d(v,v) = 0$ for all $v$.
\section{Structural Results for Optimal Embeddings}

In this section, we present two results regarding optimal layouts of a given graph. In particular, we provide a lower bound for the energy of a graph layout and an upper bound for the diameter of an optimal layout. The techniques used primarily involve estimating different components of the objective function $E(\vx_1, \ldots, \vx_n)$ given by (\ref{eqn:kk}) (written as $E(\vx)$ in this section for convenience). For this reason, we introduce the notation
\begin{align*}
E_{i,j}(\vx)&:= \bigg(\frac{\|\vx_i - \vx_j\|}{d(i,j)} - 1\bigg)^2 \quad \text{for } i,j \in [n],\\
E_{S}(\vx) &:= \sum_{\substack{i,j \in S \\ i<j}} E_{i,j}(\vx) \quad \text{for } S \subset [n], \\
E_{S,T}(\vx) &:= \sum_{i \in S} \sum_{j \in T} E_{i,j}(\vx)\quad \text{for } S,T \subset [n], \; S\cap T = \emptyset.
\end{align*}
We also make use of this notation in Appendices \ref{app:struct} and \ref{app:hardness}. First, we recall the following standard $\epsilon$-net estimate.

\begin{lemma}[Corollary 4.2.13 of \cite{vershynin2018high}]\label{lem:eps-net-ball}
Let $B_R = \{x : \|x\| \le R\} \subset \mathbb{R}^r$ be the origin-centered radius $R$ ball in $r$ dimensions. For any $\epsilon \in (0,R)$ there exists a subset $S_{\epsilon} \subset B_R$ with $|S_{\epsilon}| \le (3R/\epsilon)^r$ such that
\[ \max_{\|x\| \le R} \min_{y \in S_{\epsilon}} \|x - y\| \le \epsilon, \]
i.e. $S_{\epsilon}$ is an $\epsilon$-net of $B_R$.
\end{lemma}

Using this result, we prove the following lower bound for the objective value of any layout of a diameter $D$ graph in $\mathbb{R}^r$.

\begin{lemma}\label{lm:energy}
Let $G = ([n],E)$ have diameter 
$$D \le \frac{(n/2)^{1/r} }{10}.$$
Then any layout $\vec x \in \mathbb{R}^{rn}$ has energy 
$$E(\vec x) \ge \frac{n^2}{81(10D)^r}.$$
\end{lemma}

\begin{proof}
Let $G=([n],E)$ have diameter $D\le (n/2)^{1/r} /10$, and suppose that there exists a layout $\vec x \subset \mathbb{R}^r$ of $G$ in dimension $r$ with energy $E(\vec x) = c n^2$ for some $c \le 1/810$. If no such layout exists, then we are done. We aim to lower bound the possible values of $c$. For each vertex $i \in [n]$, we consider the quantity $E_{i, V\setminus i}(\vx)$. The sum
$$\sum_{i \in [n]} E_{i, V\setminus i}(\vx) = 2 c n^2,$$
and so there exists some $i' \in [n]$ such that 
$E_{i', V\setminus i'}(\vx) \le 2cn$. By Markov's inequality,
$$\big| \{ j \in [n] \, | \, E_{i',j}(\vx) > 10c\} \big| < n/5,$$
and so at least $4n/5$ vertices (including $i'$) in $[n]$ satisfy
$$ \left(\frac{\|\vx_{i'} -\vx_j\|}{d(i',j)} -1 \right)^2 \le 10c,$$
and also
$$\|\vx_{i'} - \vx_j\| \le d(i',j) (1 +  \sqrt{10 c}) \le \frac{10}{9} D.$$
The remainder of the proof consists of taking the $d$-dimensional ball with center $\vx_{i^*}$ and radius $10D/9$ (which contains $\ge 4n/5$ vertices), partitioning it into smaller sub-regions, and then lower bounding the energy resulting from the interactions between vertices within each sub-region.

By applying Lemma \ref{lem:eps-net-ball} with $R:= 10D/9$ and $\epsilon:=1/3$, we may partition the $r$ dimensional ball with center $\vx_{i'}$ and radius $10D/9$ into $(10 D)^r$ disjoint regions, each of diameter at most $2/3$. For each of these regions, we denote by $S_j \subset [n]$, $j \in [(10D)^r]$, the subset of vertices whose corresponding point lies in the corresponding region. As each region is of diameter at most $2/3$ and the graph distance between any two distinct vertices is at least one, either
$$E_{S_j}(\vx ) \ge {|S_j| \choose 2} (2/3 - 1)^2 = \frac{|S_j|(|S_j|-1)}{18}$$
or $|S_j| = 0$. Empty intervals provide no benefit and can be safely ignored. The optimization problem
$$ \min \; \sum_{k=1}^\ell m_k(m_k-1) \quad 
\text{s.t.} \; \sum_{k=1}^\ell m_k = m, \;m_k \ge 1, \; k \in [\ell],$$
 has a non-empty feasible region for $m \ge \ell$, and the solution is given by $m(m/\ell -1)$ (achieved when $m_k = m/\ell$ for all $k$). In our situation, $m := 4n/5$ and $\ell :=(10D)^r$, and, by assumption, $m \ge \ell$. This leads to the lower bound
 $$ cn^2 = E(\vx)  \ge \sum_{j=1}^\ell E_{S_j}(\vx ) \ge \frac{4n}{90} \left[\frac{4n}{5(10D)^r } -1 \right],$$
 which implies that
 $$c \ge \frac{16}{450(10D)^r}\left(1 - \frac{5(10D)^r}{4n}\right) \ge \frac{1}{75(10D)^r}$$
 for $D \le (n/2)^{1/r}/10$. This completes the proof.
\end{proof}

The above estimate has the correct dependence for $r=1$. For instance, consider the lexicographical product of a path $P_D$ and a clique $K_{n/D}$: i.e. a graph with $D$ cliques in a line, and complete bipartite graphs between neighboring cliques. This graph has diameter $D$, and the layout in which the ``vertices'' (each corresponding to a copy of $K_{n/D}$) of $P_D$ lie exactly at the integer values $[D]$ has objective value $\frac{n}{2}(n/D-1)$. This estimate is almost certainly not tight for dimensions $r>1$, as there is no higher dimensional analogue of the path (i.e., a graph with $O(D^r)$ vertices and diameter $D$ that embeds isometrically in $\mathbb{R}^r$).



Next, we provide an upper bound for the diameter of any optimal layout of a diameter $D$ graph. For the sake of space, the proof of this result is reserved for Appendix A.

\begin{lemma}[Proved in Appendix A] \label{lm:diam}
Let $G=([n],E)$ have diameter $D$. Then, for any optimal layout $\vec x \in \R^{rn}$, i.e., $\vx$ such that $E(\vx) \le E(\vec{y})$ for all $\vec{y} \in \mathbb{R}^{rn}$, 
$$\| \vx_i - \vec x_j \|_2 \lesssim D \log \log D$$
for all $i,j \in [n]$.
\end{lemma}

While the above estimate is sufficient for our purposes, we conjecture that this is not tight, and that the diameter of an optimal layout of a diameter $D$ graph is always at most $2D$.

\section{Algorithmic Lower Bounds}\label{sec:hardness}

In this section, we discuss algorithmic lower bounds for multidimensional scaling. In particular, we provide a sketch of the reduction used in the proof of Theorem \ref{thm:hardness-intro}. The formal proof itself is quite involved, and is therefore reserved for Appendix \ref{app:hardness}.



To show that minimizing (\ref{eqn:kk}) is NP-hard in dimension $r=1$, we use a reduction from a version of Max All-Equal 3SAT. The Max All-Equal 3SAT decision problem asks whether, given variables $t_1,\dots,t_\ell$, clauses $C_1,\dots,C_m \subset \{t_1,\dots,t_\ell,\bar{t}_1,\dots,\bar{t}_\ell\}$ each consisting of at most three literals (variables or their negation), and some value $L$, there exists an assignment of variables such that at least $L$ clauses have all literals equal. The Max All-Equal 3SAT decision problem is known to be APX-hard, as it does not satisfy the conditions of the Max CSP classification theorem for a polynomial time optimizable Max CSP  \cite{khanna2001approximability}. More precisely, this is because of the following properties: 1) setting all variables true or all variables false does not satisfy all clauses, and 2) all clauses cannot be written in disjunctive normal form as two terms, one with all unnegated variables and one with all negated variables.

We require a much more restrictive version of this problem. In particular, we require a version in which all clauses have exactly three literals, no literal appears in a clause more than once, the number of copies of a clause is equal to the number of copies of its complement (defined as the negation of all its elements), and each literal appears in exactly $k$ clauses. This more restricted version is shown to still be APX-hard in Appendix \ref{app:hardness}.

Suppose we have an instance of the aforementioned version of Max All-Equal 3SAT with variables $t_1,\dots,t_\ell$ and clauses $C_1,\dots,C_{2m}$. Let $\mathcal{L} = \{t_1,\dots,t_\ell,\bar{t}_1,\dots,\bar{t}_\ell\}$ be the set of literals and $\mathcal{C} =\{C_1,\dots,C_{2m}\}$ be the multiset of clauses. Consider the graph $G = (V,E)$, with $V = V_0 \sqcup V_1 \sqcup V_2$, where
\begin{align*}
    V_0 &= \{v^i : i\in [N_v] \}, \\
    V_1 &= \{ t^i : t \in \mathcal{L}, i \in [N_t] \}, \\
    V_2 &= \{ C^i : C \in \mathcal{C}, i \in [N_c] \},
\end{align*}
and $ E = V^{(2)} \setminus (\overline{E}_1 \cup \overline{E}_2)$, where
\begin{align*}
    \overline{E}_1 &=\{(t^i,\bar{t}^j) : t\in \mathcal{L}, i,j \in [N_t] \} \}, \\
    \overline{E}_2 &= \{(t^i,C^j) : t \in C, C \in \mathcal{C}, i \in [N_t],j \in [N_c] \}, 
\end{align*}
$\sqcup$ denotes disjoint union, parameters $N_v \gg N_t \gg N_c \gg m$, and $V^{(2)}:= \{ U \subset V \, : \, |U| =2 \}$. For simplicity, in the following description we assume that cliques (other than $V_0$) in the original graph generally embed together as one collection of nearby points, so we can treat them as single objects in the embedding. In Appendix \ref{app:hardness}, this intuition is rigorously justified.

The clique on vertices $V_0$ serves as an ``anchor'' that forces all other vertices to be almost exactly at the correct distance from its center. Without loss of generality, assume this anchor clique is centered at $0$. In this graph, the cliques corresponding to literals and clauses, given by $\{ t^i \}_{ i \in [N_t]}$ and $\{ C^i \}_{i \in [N_C]}$ respectively, are all at distance one from the anchor clique. Literal cliques are at distance one from each other, except negations of each other, which are at distance two. Clause cliques are distance two from the literal cliques corresponding to literals in the clause and distance one from literal cliques corresponding to literals not in the clause. Clause cliques are all distance one from each other. The main idea of the reduction is that the location of the center of the anchor clique at $0$ forces each literal to roughly be at either $-1$ or $+1$, and the distance between negations forces negations to be on opposite sides, i.e., $\vec x_{t^i} \approx - \vec x_{\bar t^i}$. Clause cliques are also roughly at either $-1$ or $+1$ and the distance to literals forces clauses to be opposite the side with the majority of its literals, i.e., clause $C = \{t_1,t_2,t_3\}$ lies at 
$\vec x_{C^i} \approx - \chi\{\vec x_{t_1^i} + \vec x_{t_2^i} + \vec x_{t_3^i} \ge 0\}$, where $\chi$ is the indicator variable. The optimal embedding of $G$, i.e. the location of variable cliques at either $+1$ or $-1$, corresponds to an optimal assignment for the Max All-Equal 3SAT instance.

\begin{remark}[Comparison to \cite{cayton2006robust}]\label{rmk:differences}
As mentioned in the Introduction, the reduction here is significantly more involved than the hardness proof for a related problem in \cite{cayton2006robust}. At a high level, the key difference is that in \cite{cayton2006robust} they were able to use a large number of distance-zero vertices to create a simple structure around the origin. This is no longer possible in our setting (in particular, with bounded diameter graph metrics), which results in graph layouts with much less structure. For this reason, we require a graph that exhibits as much structure as possible. To this end, a reduction from Max All-Equal 3SAT using both literals and clauses in the graph is a much more suitable technique than a reduction from NAE 3SAT using only literals. In fact, it is not at all obvious that the same approach in \cite{cayton2006robust}, applied to unweighted graphs, would lead to a computationally hard instance.
\end{remark}

\section{Approximation Algorithm}\label{sec:approx}
In this section, we formally describe an approximation algorithm using tools from the Dense CSP literature, and prove theoretical guarantees for the algorithm. 
\subsection{Preliminaries: Greedy Algorithms for Max-CSP}
A long line of work studies the feasibility of solving the Max-CSP problem under various related pseudorandomness and density assumptions.
In our case, an algorithm with mild dependence on the alphabet size is extremely important. A very simple greedy approach, proposed and analyzed by Mathieu and Schudy \cite{mathieu2008yet,schudy} (see also \cite{yaroslavtsev2014going}), satisfies this requirement.
\begin{algorithm}
\caption{Greedy Algorithm for Dense CSPs \cite{mathieu2008yet,schudy}}
\begin{algorithmic}[1]
\FUNCTION{GreedyCSP$(\Sigma,n,t_0,\{f_{ij}\})$\label{alg:schudy}}
\STATE Shuffle the order of variables $x_1,\ldots,x_n$ by a random permutation.
\FOR{all assignments $x_1,\ldots,x_{t_0} \in \Sigma^{t_0}$}
\FOR{$(t_0 + 1) \le i \le n$}
\STATE Choose $x_i \in \Sigma$ to maximize
\[ \sum_{j < i} f_{ji}(x_j,x_i) \]
\ENDFOR
\STATE Record $x$ and objective value $\sum_{i \ne j} f_{ij}(x_i,x_j)$.
\ENDFOR
\STATE Return the assignment $x$ found with maximum objective value. 
\ENDFUNCTION
\end{algorithmic}
\end{algorithm}
\begin{theorem}[\cite{mathieu2008yet,schudy}]\label{thm:dense-csp}
Suppose that $\Sigma$ is a finite alphabet, $n \ge 1$ is a positive integer, and for every $i,j \in {n \choose 2}$ we have a function $f_{ij} : \Sigma \times \Sigma \to [-M,M]$. Then for any $\epsilon > 0$, Algorithm~\textsc{GreedyCSP} with $t_0 = O(1/\epsilon^2)$ runs in time $n^2 |\Sigma|^{O(1/\epsilon^2)}$ and returns $x_1,\ldots,x_n \in \Sigma$ such that 
\[ \mathbb{E} \sum_{i \ne j} f_{ij}(x_i,x_j) \ge \sum_{i \ne j} f_{ij}(x^*_i,x^*_j) - \epsilon M n^2  \]
for any $x^*_1,\ldots,x^*_n \in \Sigma$, where $\mathbb{E}$ denotes the expectation over the randomness of the algorithm. 
\end{theorem}
\noindent
In comparison, we note that computing the maximizer using brute force would run in time $|\Sigma|^{n}$, i.e. exponentially slower in terms of $n$.
This guarantee is stated in expectation but, if desired, can be converted to a high probability guarantee by using Markov's inequality and repeating the algorithm multiple times (as in Remark~\ref{rmk:variants}). We use \textsc{GreedyCSP} to solve a minimization problem instead of maximization, which corresponds to negating all of the functions $f_{ij}$.

\subsection{Discretization Argument}
\begin{lemma}\label{lem:lipschitz1}
For $c,R > 0$, the function $x \mapsto (x/c - 1)^2$ is $\frac{2}{c} \max(1, R/c)$-Lipschitz on the interval $[0,R]$.
\end{lemma}
\begin{proof}
Because the derivative of the function is $\frac{2}{c} (x/c - 1)$ and $\left|\frac{2}{c}(x/c - 1)\right| \le \frac{2}{c} \max(1, R/c)$ on $[0,R]$, the result
follows from the mean value theorem.
\end{proof}
\begin{lemma}\label{lem:lipschitz2}
For $c,R > 0$ and $y \in \mathbb{R}^r$ with $\|y\| \le R$, the function $x \mapsto (\|x - y\|/c - 1)^2$ is $\frac{2}{c} \max(1,2R/c)$-Lipschitz on $B_R = \{ x : \|x\| \le R\}$.
\end{lemma}
\begin{proof}
Because the function $\|x - y\|$ is $1$-Lipschitz and $\|x - y\| \le \|x\| + \|y\| \le 2R$ by the triangle inequality, the result follows from Lemma~\ref{lem:lipschitz1} and the fact that a composition of Lipschitz functions is Lipschitz.
\end{proof}

\begin{lemma}\label{lem:discretization-main}
Let $\vx_1,\ldots,\vx_n \in \mathbb{R}^r$ be arbitrary vectors such that $\|\vx_i\| \le R$ for all $i$ and $\epsilon > 0$ be arbitrary. Define $S_{\epsilon}$ to be an $\epsilon$-net of $B_R$ as in Lemma~\ref{lem:eps-net-ball}, so $|S_{\epsilon}| \le (3R/\epsilon)^r$. For any input metric over $[n]$ with $\min_{i,j \in [n]} d(i,j) = 1$, there exists $\vx'_1,\ldots,\vx'_n \in S_{\epsilon}$ such that
\[ E(\vx'_1,\ldots,\vx'_n) \le E(\vx_1,\ldots,\vx_n) + 4\epsilon R n^2  \]
where $E$ is \eqref{eqn:kk} defined with respect to an arbitrary graph with $n$ vertices.
\end{lemma}
\begin{proof}
By Lemma~\ref{lem:lipschitz2} the energy $E(\vx_1,\ldots,\vx_n)$ is the sum of ${n \choose 2} \le n^2/2$ many terms, which, for each $i$ and $j$, are individually $4R$-Lipschitz in $\vx_i$ and $\vx_j$. Therefore, defining $\vx'_i$ to be the closest point in $S_{\epsilon}$ for all $i$ gives the desired result. 
\end{proof}
\subsection{Approximation Algorithm}
\begin{algorithm}
\caption{Approximation Algorithm \textsc{KKScheme}}
\begin{algorithmic}[1]
\FUNCTION{\textsc{KKScheme($\epsilon_1,\epsilon_2,R$)}:}
\STATE Build an $\epsilon_1$-net $S_{\epsilon_1}$ of $B_R = \{x : \|x\| \le R\} \subset \mathbb{R}^r$ as in Lemma~\ref{lem:eps-net-ball}.
\STATE Apply the \textsc{GreedyCSP} algorithm of Theorem~\ref{thm:dense-csp} with $\epsilon = \epsilon_2$ to approximately minimize $E(\vx_1,\ldots,\vx_n)$ over $\vx_1,\ldots,\vx_n \in S_{\epsilon_1}^n$.
\STATE Return $\vx_1,\ldots,\vx_n$.
\ENDFUNCTION
\end{algorithmic}
\end{algorithm}
\begin{theorem}[Formal Statement of Theorem~\ref{thm:kk-alg-intro}]\label{thm:kk-alg}
Let $R > \epsilon > 0$ be arbitrary. For any input metric over $[n]$ with $\min_{i,j \in [n]} d(i,j) = 1$, Algorithm~\textsc{KKScheme} with $\epsilon_1 = O(\epsilon/R)$ and $\epsilon_2 = O(\epsilon/R^2)$ runs in time $n^2 (R/\epsilon)^{O(rR^4/\epsilon^2)}$ and outputs $\vx_1,\ldots,\vx_n \in \mathbb{R}^r$ with $\|\vx_i\| \le R$ such that
\[ \mathbb{E}\left[E(\vx_1,\ldots,\vx_n)\right] \le E(\vx^*_1,\ldots,\vx^*_n) + \epsilon n^2 \]
for any $\vx^*_1,\ldots,\vx^*_n$ with $\|\vx^*_i\| \le R$ for all $i$, where $\mathbb{E}$ is the expectation over the randomness of the algorithm.
\end{theorem}
\begin{proof}
By combining Lemma~\ref{lem:discretization-main} with Theorem~\ref{thm:dense-csp} (used as a minimization instead of maximization algorithm), the output $\vx_1,\ldots,\vx_n$ of \textsc{KKScheme} satisfies
\[ E(\vx_1,\ldots,\vx_n) \le E(\vx^*_1,\ldots,\vx^*_n) + 4\epsilon_1 R n^2 + \epsilon_2 R^2 n^2 \]
and runs in time $n^2 2^{O(1/\epsilon_2^2) r\log(3R/\epsilon_1)}$. Taking $\epsilon_2 = O(\epsilon/R^2)$ and $\epsilon_1 = O(\epsilon/R)$ gives the desired result.
\end{proof}
\begin{remark}\label{rmk:variants}
The runtime can be improved to $n^2 + (R/\epsilon)^{O(dR^4/\epsilon^2)}$ using a slightly more complex greedy CSP algorithm \cite{mathieu2008yet}. Also, by the usual argument, a high probability guarantee can be derived by repeating the algorithm $O(\log(2/\delta))$ times, where $\delta > 0$ is the desired failure probability.
\end{remark}
\subsection{Extension to Vertex-Weighted Setting}
In this section, we generalize the approximation algorithm to handle vertex weights. This generalization is useful if vertices have associated \emph{importance weights}, e.g. each vertex represents a different number of people, and larger/more important vertices should be embedded more accurately. Given a probability measure $\mu$ over $[n]$, the weighted Kamada-Kawai objective is
\begin{equation}\label{eqn:kk-weighted}
E^{\mu}(\vx_1,\ldots,\vx_n) = n^2 \sum_{i < j} \mu(i) \mu(j) \left(\frac{\|\vx_i - \vx_j\|}{d(i,j)} - 1\right)^2.
\end{equation}
Note that when $\mu$ is the uniform measure on $[n]$, this reduces to \eqref{eqn:kk}. 

\begin{theorem}\label{thm:kk-alg-weighted}
Let $R > \epsilon > 0$ be arbitrary. Algorithm~\textsc{KKScheme} with $\epsilon_1 = O(\epsilon/R)$ and $\epsilon_2 = O(\epsilon/R^2)$ runs in time $n^{O(rR^4\log(R/\epsilon)/\epsilon^2)}$ and outputs $\vx_1,\ldots,\vx_n \in \mathbb{R}^r$ with $\|\vx_i\| \le R$ such that
\[ \mathbb{E}\left[E(\vx_1,\ldots,\vx_n)\right] \le E(\vx^*_1,\ldots,\vx^*_n) + \epsilon n^2 \]
for any $\vx^*_1,\ldots,\vx^*_n$ with $\|\vx^*_i\| \le R$ for all $i$, where $\mathbb{E}$ is the expectation over the randomness of the algorithm.
\end{theorem}
\begin{proof}
The proof is the same as Theorem~\ref{thm:kk-alg}, except that we require a different dense CSP algorithm. More precisely, we can directly verify that the discretization Lemma, Lemma~\ref{lem:discretization-main}, holds with the same guarantee for the weighted Kamada-Kawai objective. This reduces the problem to approximating a dense CSP with vertex weights, for which we use Theorem~\ref{thm:dense-csp-weighted}.
\end{proof}
The following Theorem formally describes the guarantee we use for approximately optimizing dense CSPs with vertex/variable weights. This result can be proved by slightly modifying the algorithm and analysis in \cite{yoshida2014approximation}. For completeness, we provide a proof in Appendix~\ref{apdx:vertex-weighted-csp}.
\begin{theorem}[Proved in Appendix~\ref{apdx:vertex-weighted-csp}]\label{thm:dense-csp-weighted}
Suppose that $\Sigma$ is a finite alphabet, $n \ge 1$ is a positive integer, and for every $i,j \in {n \choose 2}$ we have a function $f_{ij} : \Sigma \times \Sigma \to [-M,M]$. Then for any $\epsilon > 0$, there exists an algorithm which runs in time $n^{O(\log |\Sigma|/\epsilon^2)}$ and returns $x_1,\ldots,x_n \in \Sigma$ such that 
\[ \mathbb{E}\left[\mathbb{E}_{i,j \sim \mu} f_{ij}(x_i,x_j)\right] \ge \mathbb{E}_{i,j \sim \mu} f_{ij}(x^*_i,x^*_j) - \epsilon M  \]
for any $x^*_1,\ldots,x^*_n \in \Sigma$, where the outer $\mathbb{E}$ denotes the expectation over the randomness of the algorithm. 
\end{theorem}
\section{Experiments}
We implemented the \textsc{GreedyCSP}-based algorithm described above as well as a standard gradient descent approach to minimizing the Kamada-Kawai objective. In this section we compare the behavior of these algorithms in a few interesting instances. In addition to gradient descent, a couple of other local search heuristics are popular for minimizing the Kamada-Kawai objective: 1) the original algorithm proposed by Kamada and Kawai \cite{kamada1989algorithm}, which updates single points at a time using a Newton-Raphson scheme, and 2) a variational approach known as \emph{majorization}, which optimizes a sequence of upper bounds on the KK objective \cite{de1977applications,gansner2004graph}, where each step reduces to solving a Laplacian system. The recent works \cite{zheng2018graph,borsig2020stochastic} compared these local search heuristics and argued that (stochastic) gradient descent, proposed in the early work of \cite{kruskal1964multidimensional}, is one of the best performing methods in practice. For this reason, we focus on comparing with gradient descent.

\subsubsection{Some Graph Drawing Examples.} 
In Figure~\ref{fig:watts_strogatz} we show the result of embedding a random Watts-Strogatz ``small world''  graph \cite{watts1998collective}, a model of random graph intended to reflect some properties of real world networks. In Figure~\ref{fig:elt3} we show an embedding of the ``3elt'' graph from \cite{graphs}; in this case, it's interesting that all of the methods optimizing \eqref{eqn:kk} seem to find the same solution, except Greedy suffers a small loss due to discretization. This suggests that this solution may be the global optimum.

Note that in all figures, the MDS/Kamada-Kawai objective value achieved (normalized by $1/n^2$, where $n$ is the number of vertices) is included in the subtitle of each plot.  For comparison, in the bottom right of each Figure we display the standard spectral embedding given by embedding each vertex according to the entries of the bottom two nontrivial eigenvectors of the graph Laplacian.
\begin{figure}
    \centering
    \includegraphics[scale=0.5,trim=0 0.1 0 0,clip]{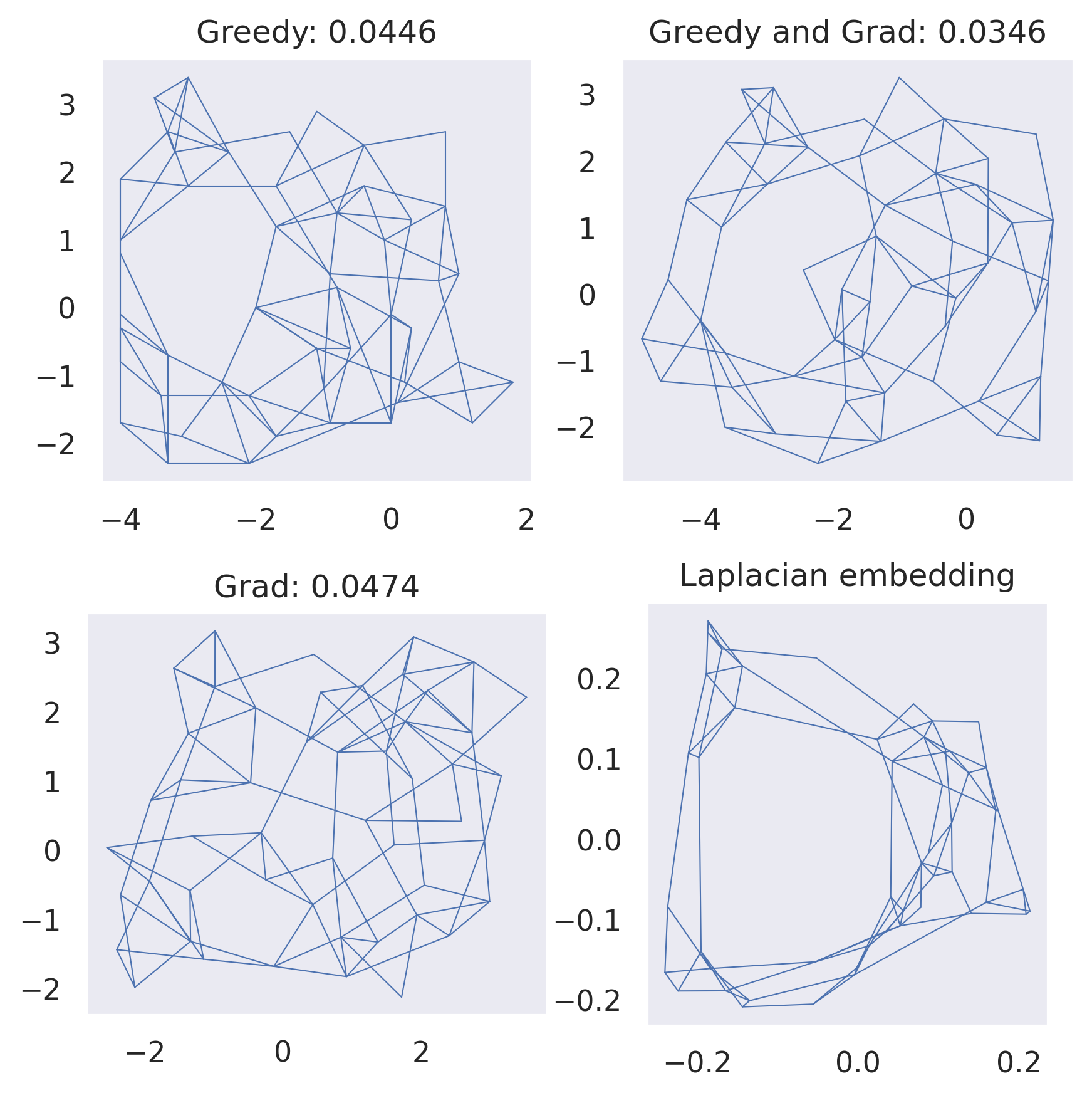}
    \caption{Embeddings of Watts Strogatz graph on 50 nodes with graph parameters $K = 4$ and $\beta = 0.3$ and $t_0 = 3$ for \textsc{GreedyCSP}. }
    \label{fig:watts_strogatz}
\end{figure}

\begin{figure}
    \centering
    \includegraphics[scale=0.5]{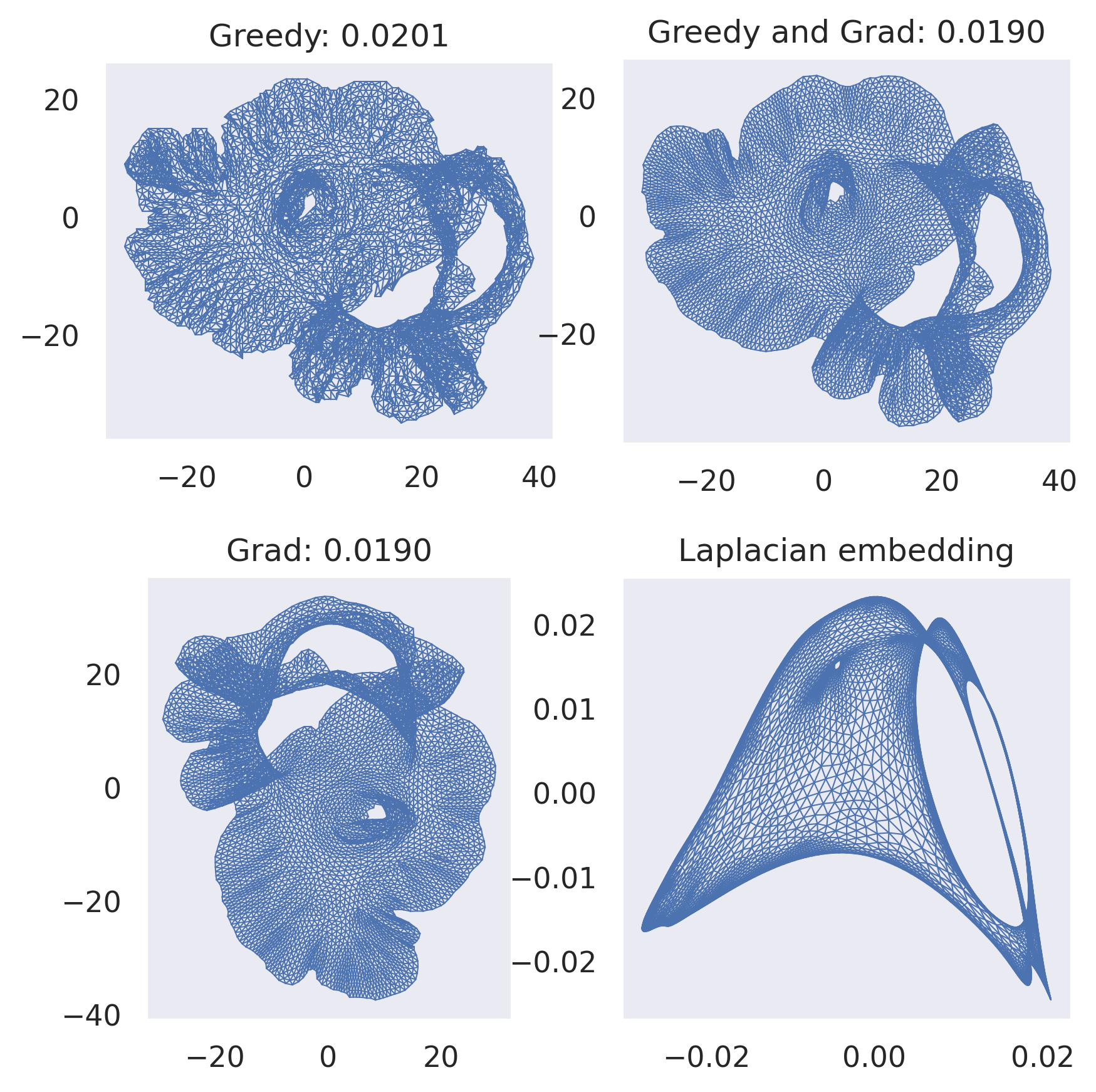}    
    \caption{3elt graph from AG Monien collection \cite{graphs}; \textsc{GreedyCSP} run with parameter $t_0 = 2$.}
    \label{fig:elt3}
\end{figure}

\subsubsection{Experiment with restarts.} The algorithm we propose in Theorem~\ref{thm:kk-alg} is randomized, which leaves open the possibility that better results are obtained by running the algorithm multiple times and taking the best result.
In Figure~\ref{fig:davis}, we show the result of embedding a well-known social network graph, the Davis Southern Women Network \cite{davis2009deep}, by running all methods 10 times and taking the result with best objective value. This graph has a total of 32 nodes and records the attendance of 18 Southern women at 14 social events during the 1930s. To compare with the minimum, the average objective value achieved in the run is 0.0588, 0.0498, and 0.0515 for Greedy, Greedy and Grad, and Grad respectively so all methods did improved slightly by running multiple times. Finally, we note that running gradient descent with 30 restarts (as opposed to 10) improved its best score to $0.0478$, essentially the same as the Greedy and Grad result.

\begin{figure}
    \centering
    \includegraphics[scale=0.5,trim=0 5 0 0,clip]{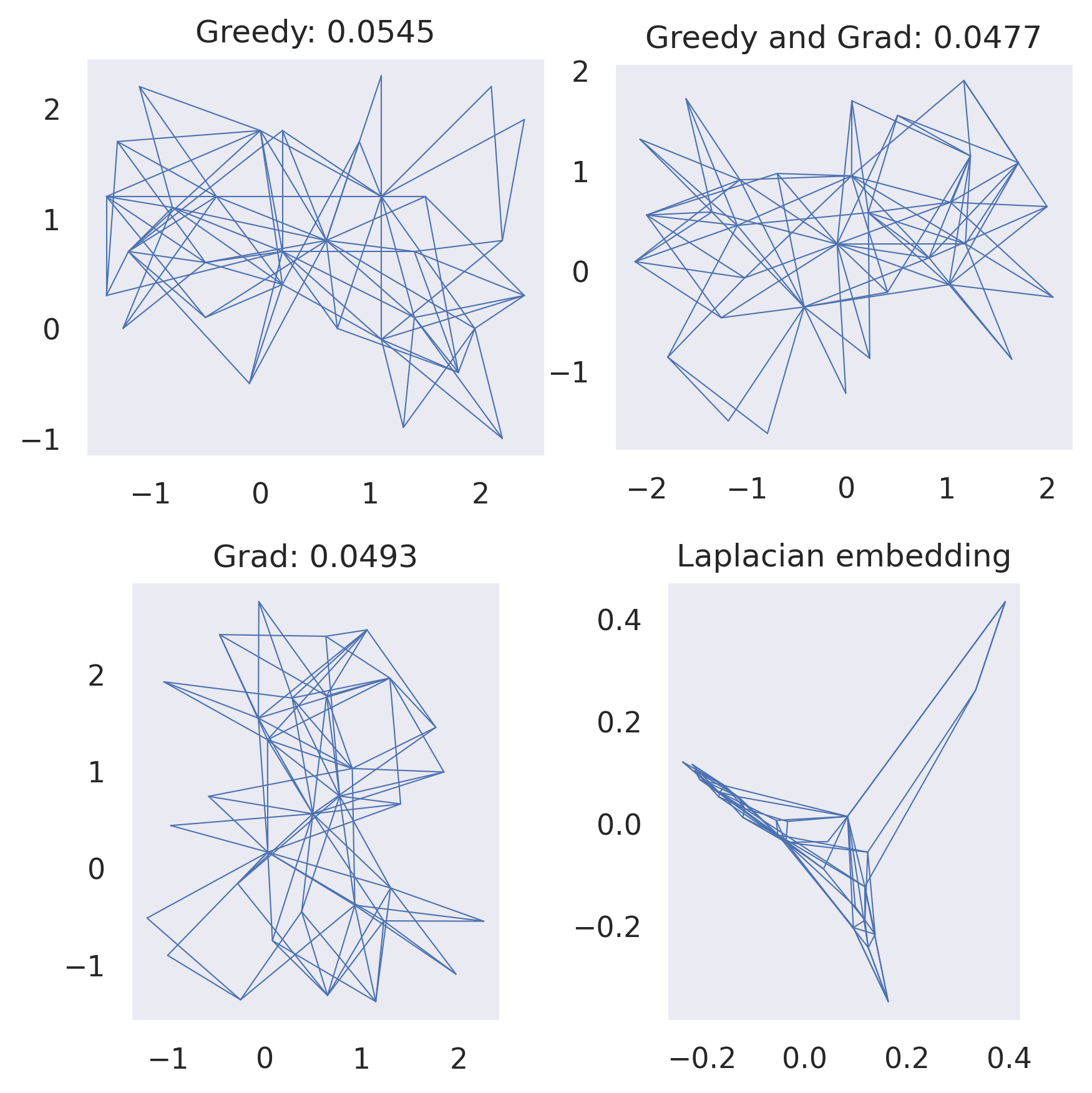}
    \caption{Embedding of Davis Southern Women Network graph. The top left figure was generated using \textsc{GreedyCSP} with $t_0 = 3$.}
    \label{fig:davis}
\end{figure}

\subsubsection{Community Detection Experiment.} A lot of the recent interest in force-directed graph drawing algorithms has been in their ability to discover interesting latent structure in data and with a view towards applications like non-linear dimensionality reduction. As a test of this concept on synthetic data, we tested the algorithms on a celebrated model of latent community structure in graphs, the stochastic block model. The results are shown in Figure~\ref{fig:sbm}, along with the results of a standard spectral embedding using the bottom two nontrivial eigenvectors of the Laplacian. We did not draw the edges in this case as they make the Figure difficult to read; more importantly, the location of points in the embedding show that nontrivial community structure was recovered; for example, the green and blue communities are roughly linearly separable in all of the embeddings.
Note that the spectral embedding approach admits strong provable guarantees for community recovery (see the survey \cite{abbe2017community}), and so the interesting thing to observe here is that the force-directed drawing methods also recover nontrivial information about the latent structure.

\begin{table}
    \centering
    \begin{tabular}{l|c|c|c}
    Runtime & Greedy & Grad & Laplacian \\
    \hline
    Davis &  5.6 s & 4 s & 4 ms \\
    Watts-Strogatz  & 453 s & 4 s & 20 ms 
    \end{tabular}
    \caption{Runtimes for methods with parameters used in figures.}
    \label{tab:my_label}
\end{table}

\subsubsection{Implementation details.} All experiments were performed on a standard Kaggle GPU Kernel with a V80 GPU. Gradient descent was run with learning rate $0.005$ for $4000$ steps on all instances. We seeded the RNG with zero before each simulation for reproducibility.
For the greedy method, we eliminated the rotation and translation degrees of freedom when implementing the initial brute force step; the  parameter $R$ was set to $2.5$ for the Davis experiment, and set to $4$ for all others --- informally, the tuning rule for this parameter is to increase its value until the plot does not hit the boundary of the region. We compare runtimes in Table~\ref{tab:my_label}; the runtime for Greedy in Watts-Strogatz is much larger due to the larger value of $n$ and of $R$ used; the latter roughly corresponds to the larger diameter of the underlying graph (cf. Lemma~\ref{lm:energy}).
\begin{figure}
    \centering
    \includegraphics[scale=0.5]{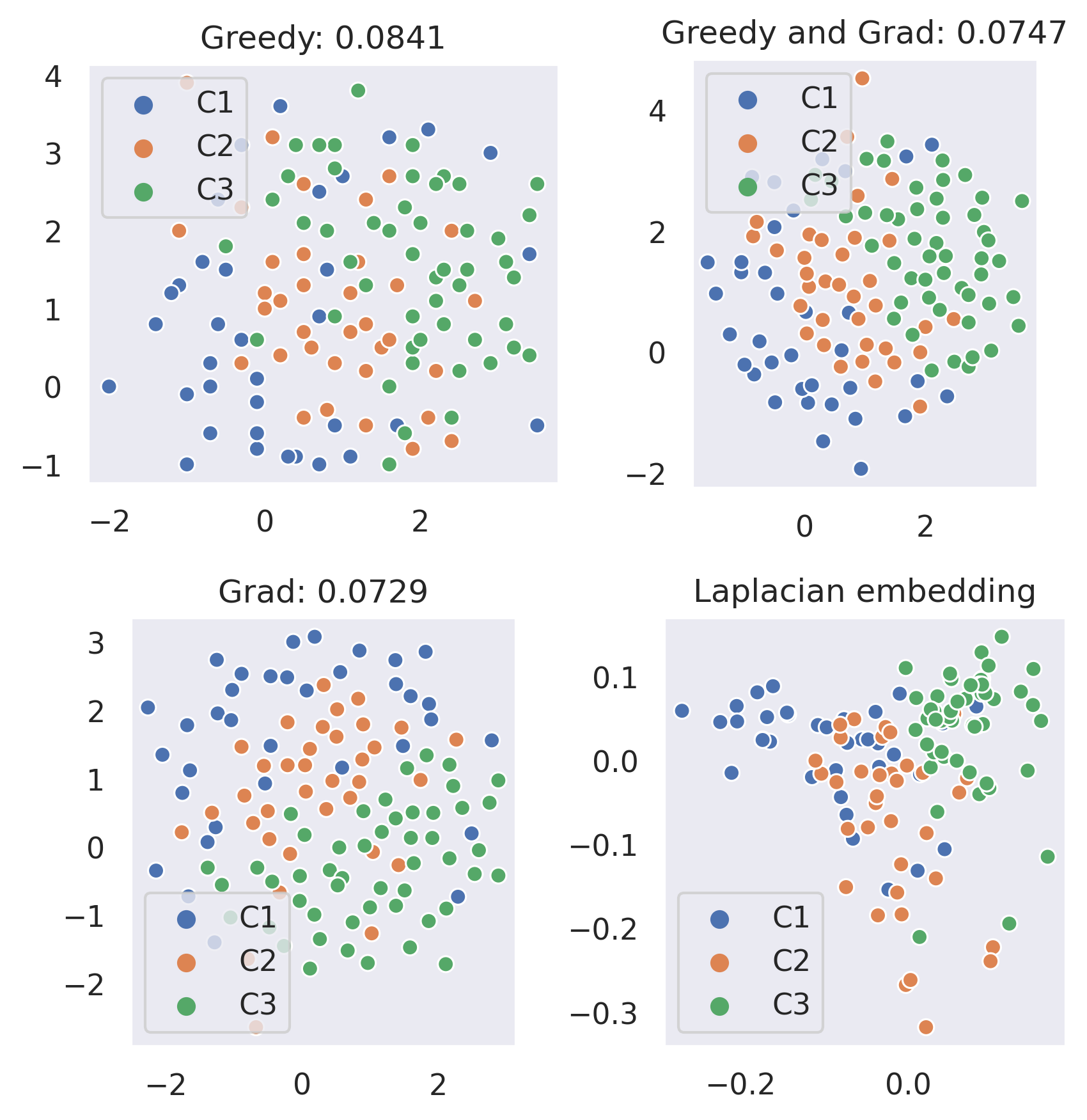}
     \caption{Embeddings of a 3-community Stochastic Block Model (SBM) with connection probabilities $\left(\protect\begin{smallmatrix}0.09 & 0.03 & 0.02 \\
             0.03 & 0.15 & 0.04 \\
             0.02 & 0.04 & 0.1\protect\end{smallmatrix}\right)$ and community sizes $35,35,50$. Colors correspond to latent community assignments. The top left is constructed using the \textsc{GreedyCSP} algorithm with $t_0 = 3$. For this experiment only, we used the degree-normalized Laplacian since it is generally preferred in the context of the SBM.}
    \label{fig:sbm}
\end{figure}
\section{Conclusions}
Our theory and experimental results suggest the following natural question: does gradient descent, with enough random restarts, have a similar provable guarantee to Theorem~\ref{thm:dense-csp}? As noted in our experiments and in the experiments of \cite{zheng2018graph}, gradient-based optimization often seems to find high quality (albeit not global) minima of the Kamada-Kawai objective, even though the loss is highly non-convex. In fact, combining our analysis with a different theorem from \cite{schudy} proves that running a variant of \textsc{GreedyCSP} without the initial brute force step (i.e. with $t_0 = 0$), achieves an additive $O(\epsilon n^2)$ approximation if we repeat the algorithm $2^{2^{1/\epsilon^2}}$ many times. A similar guarantee for gradient descent, a different sort of greedy procedure, sounds plausible. 

\section*{Acknowledgements}
Frederic Koehler was supported in part by NSF CAREER Award CCF-1453261, NSF Large CCF-1565235, Ankur Moitra's ONR Young Investigator Award, and E. Mossel's Vannevar Bush Fellowship ONR-N00014-20-1-2826. The work of J. Urschel was supported in part by ONR Research Contract N00014-17-1-2177. The authors would like to thank Michel Goemans for valuable conversations on the subject. The authors are grateful to Louisa Thomas for greatly improving the style of presentation.

{ \small 
	\bibliographystyle{plain}
	\bibliography{bib} }

\newpage 
\appendix

\section{Structural Results: A Proof of Lemma \ref{lm:diam}}\label{app:struct}

In this section, we provide a proof of Lemma \ref{lm:diam}, which upper bounds the diameter of any optimal layout of a diameter $D$ graph. However, before we proceed with the proof, we first prove an estimate for the concentration of points $\vx_i$ at some distance from the marginal median in any optimal layout. The \emph{marginal median} $\vec{m} \in \mathbb{R}^r$ of a set of points $\vx_1,\ldots,\vx_n \in \mathbb{R}^r$ is the vector whose $i^{th}$ component $\vec{m}(i)$ is the univariate median of $\vx_1(i),\ldots,\vx_n(i)$ (with the convention that, if $n$ is even, the univariate median is given by the mean of the $(n/2)^{th}$ and $(n/2 +1)^{th}$ ordered values). The below result, by itself, is not strong enough to prove our desired diameter estimate, but serves as a key ingredient in the proof of Lemma \ref{lm:diam}. 

\begin{lemma}\label{lm:conc}
Let $G=([n],E)$ have diameter $D$. Then, for any optimal layout $\vec x \in \mathbb{R}^{rn}$, i.e., $\vx$ such that $E(\vx) \le E(\vec{y})$ for all $\vec{y} \in \mathbb{R}^{rn}$, 
$$\big| \{ i \in [n] \, | \, \| \vec m - \vec x_i \|_\infty \ge (C+k) D \} \big| \le 2r n C^{-\sqrt{2}^k}$$
for all $C>0$ and $k \in \mathbb{N}$, where $\vec m$ is the marginal median of $\vx_1,\ldots, \vx_n$.
\end{lemma}

\begin{proof}
Let $G = ([n],E)$ have diameter $D$, and $\vx$ be an optimal layout. Without loss of generality, we order the vertices so that $\vec x_1(1) \le \cdots \le  \vec x_n(1)$ and shift $\vec x$ so that $\vx_{\lfloor n/2 \rfloor}(1) = 0$. Next, we fix some $C > 0$, and define the subsets $S_0:= \big[ \lfloor n/2 \rfloor \big]$ and
\begin{align*}
   S_k &:= \big\{ i \in [n] \, | \, \vx_i(1) \in \big[(C+k)D,(C+k+1)D \big) \big\}, \\
   T_k &:= \big\{ i \in [n] \, | \, \vx_i(1) \ge (C+k)D \big\},
\end{align*}
$k \in \mathbb{N}$. Our primary goal is to estimate the quantity $|T_k|$, for an arbitrary $k$, from which the desired result will follow quickly.

The objective value $E(\vx)$ is at most $n \choose 2$; otherwise we could replace $\vx$ by a layout with all $\vx_i$ equal. To obtain a first estimate on $|T_k|$, we consider a crude lower bound on $E_{S_0,T_k}(\vx)$. We have
$$E_{S_0,T_k}(\vx) \ge |S_0| |T_k| \left( \frac{(C+k)D}{D} -1\right)^2 = (C+k-1)^2 \lfloor n/2 \rfloor |T_k|,$$
and therefore
$$|T_k| \le {n \choose 2} \frac{1}{(C+k-1)^2 \lfloor n/2 \rfloor} \le \frac{n}{(C+k-1)^2}.$$

From here, we aim to prove the following claim,
\begin{equation}\label{cl:conc}
|T_k| \le \frac{n}{\big(C+k-(2\ell-1)\big)^{2^\ell}}, \quad k,\ell \in \mathbb{N}, \; k \ge 2\ell-1,
\end{equation}
by induction on $\ell$. The above estimate serves as the base case for $\ell =1$. Assuming the above statement holds for some fixed $\ell$, we aim to prove that this holds for $\ell+1$.

To do so, we consider the effect of collapsing the set of points in $S_k$, $k \ge 2 \ell$, into a hyperplane with a fixed first value. In particular, consider the alternate layout $\vx'$ given by
$$\vx'_i(j) = \begin{cases}
(C+k)D & j = 1, \, i \in S_k \\
\vx_i(j) - D & j = 1, \, i \in T_{k+1} \\
 \vx_i(j) & \text{otherwise}
\end{cases}.$$
For this new layout, we have
$$E_{S_{k-1}\cup S_k \cup S_{k+1}}(\vx') \le E_{S_{k-1}\cup S_k \cup S_{k+1}}(\vx) + 
 \frac{|S_k|^2}{2}  + \big( |S_{k-1}|+|S_{k+1}|\big)|S_k|$$
and
\begin{align*}
    E_{S_0,T_{k+1}}(\vx') &\le E_{S_0,T_{k+1}}(\vx) - |S_0| |T_{k+1}| \left( \big( (C+k+1) -1 \big)^2 - \big((C+k) -1 \big)^2 \right) \\ &=E_{S_0,T_{k+1}}(\vx) - (2C+2k-1) \lfloor n/2 \rfloor |T_{k+1}|.
\end{align*}
Combining these estimates, we note that this layout has objective value bounded above by 
\begin{align*}
    E(\vx') &\le E(\vx) + \big( |S_{k-1}|+ |S_{k}|/2 + |S_{k+1}|\big)|S_k| -  (2C+2k-1) \lfloor n/2 \rfloor |T_{k+1}| \\
    &\le E(\vx) + |T_{k-1}| |T_{k}| -  (2C+2k-1) \lfloor n/2 \rfloor |T_{k+1}|,
\end{align*}
and so
\begin{align*}
    |T_{k+1}| &\le \frac{|T_{k-1}| |T_{k}|}{(2C+2k-1) \lfloor n/2 \rfloor} \\
    &\le \frac{n^2}{\big(C+(k-1)-(2\ell-1)\big)^{2^\ell}\big(C+k-(2\ell-1)\big)^{2^\ell} (2C+2k-1) \lfloor n/2 \rfloor}\\
    &\le \frac{n}{2\lfloor n/2 \rfloor} \frac{\big(C+(k-1)-(2\ell-1)\big)^{2^\ell}}{\big(C+k-(2\ell-1)\big)^{2^\ell}}  \frac{n}{\big(C+(k-1)-(2\ell-1)\big)^{2^{\ell+1}}} \\
    &\le \frac{n}{\big(C+(k+1)-(2(\ell+1)-1)\big)^{2^{\ell+1}}}
\end{align*}
for all $k +1 \ge 2(\ell+1)-1$ and $C+k \le n+1$. If $C +k >n+1$, then
$$|T_{k+1}| \le \frac{|T_{k-1}| |T_{k}|}{(2C+2k-1) \lfloor n/2 \rfloor} \le \frac{|T_{k-1}| |T_{k}|}{(2n+1) \lfloor n/2 \rfloor} <1,$$
and so $|T_{k+1}|=0$. This completes the proof of claim (\ref{cl:conc}), and implies that $|T_{k}| \le n \, C^{-\sqrt{2}^k}$. Repeating this argument, with indices reversed, and also for the remaining dimensions $j = 2,\dots,r$, leads to the desired result
$$\big| \{ i \in [n] \, | \, \| \vec m - \vec x_i \|_\infty \ge (C+k) D \} \big| \le 2r n C^{-\sqrt{2}^k}$$
for all $k \in \mathbb{N}$ and $C>0$.
\end{proof}

Using the estimates in the proof of Lemma \ref{lm:conc} (in particular, the bound $|T_k| \le n C^{-\sqrt{2}^k}$), we are now prepared to prove Lemma \ref{lm:diam}.

\begin{lemma}[Restatement of Lemma \ref{lm:diam}]\label{lm:diam2}
Let $G=([n],E)$ have diameter $D$. Then, for any optimal layout $\vec x \in \R^{rn}$, i.e., $\vx$ such that $E(\vx) \le E(\vec{y})$ for all $\vec{y} \in \mathbb{R}^{rn}$, 
$$\| \vx_i - \vec x_j \|_2 \le 8D + 4D \log_2 \log_2 2D$$
for all $i,j \in [n]$.
\end{lemma}

\begin{proof}

Let $G = ([n],E)$, have diameter $D$, and $\vx$ be an optimal layout. Without loss of generality, suppose that the largest distance between any two points of $\vec x$ is realized between two points lying in $\text{span}\{e_1\}$ (i.e., on the axis of the first dimension). In addition, we order the vertices so that $\vec x_1(1) \le \cdots \le  \vec x_n(1)$ and translate $\vec x$ so that $\vx_{\lfloor n/2 \rfloor}(1) = 0$. 

Let $x_1(1) = -\alpha e_1$, and suppose that $\alpha \ge 5D$. If this is not the case, then we are done. By differentiating $E$ with respect to $\vec x_1(1)$, we obtain
$$\frac{\partial E}{\partial \vec x_1(1)} = -  \sum_{j=2}^n  \left( \frac{\|\vec x_j - \vec x_1\|}{d(1,j)} -1\right) \frac{
\vec x_j(1)+\alpha}{d(1,j) \|\vec x_j- \vec x_1\| }= 0.$$ 
Let
$$T_1: = \{j \in [n] \, : \, \|\vx_j - \vx_1 \| \ge d(1,j) \},$$
$$T_2: = \{j \in [n] \, : \, \|\vx_j - \vx_1 \| < d(1,j) \},$$
and compare the sum of the terms in the above equation corresponding to $T_1$ and $T_2$, respectively. We begin with the former. Note that for all $j \ge \lfloor n/2 \rfloor$ we must have $j \in T_1$, because $\|x_j - x_1\| \ge \alpha \ge 5D > d(1,j)$. 
Therefore we have
\begin{align*}
    \sum_{j\in T_1}  \left( \frac{\|\vec x_j - \vec x_1\|}{d(1,j)} -1\right) \frac{
\vec x_j(1)+\alpha}{d(1,j) \|\vec x_j- \vec x_1\| } 
    &\ge \sum_{j = \lfloor n/2\rfloor}^n \left( \frac{\| \vec x_j -\vec x_1\|}{d(1,j)} -1\right) \frac{ \vec x_j(1)+\alpha}{d(1,j) \| \vec x_j- \vec x_1\| }\\ 
    &\ge \lceil n/2 \rceil \left( \frac{\| \vec x_j- \vec x_1\|}{D} -1\right) \frac{\alpha}{D \| \vec x_j- 
    \vec x_1\| }  \\
    &\ge \frac{ \lceil n/2 \rceil (\alpha-D) }{D^2}  .
\end{align*}
where in the last inequality we used $\|x_j - x_1\| \ge \alpha$ and $\|x_j - x_1\|/D - 1 \ge \alpha/D - 1$.
Next, we estimate the sum of terms corresponding to $T_2$. Let 
$$T_3:=\{ i \in [n] \, : \, \vec x_i(1)+\alpha \le D \}$$
and note that $T_2 \subset T_3$.
We have
\begin{align*}
    \sum_{j\in T_2}  \left(1- \frac{\|\vec x_j - \vec x_1\|}{d(1,j)} \right) \frac{
\vec x_j(1)+\alpha}{d(1,j) \|\vec x_j- \vec x_1\| } 
&= \sum_{j \in T_3} \left| 1- \frac{\| \vec x_j - \vec x_1\|}{d(1,j)} \right|_{+} \frac{ \vec x_j(1)+\alpha}{d(1,j) \| \vec x_j- \vec x_1\| } \\ 
    &\le \sum_{j \in T_3} \left| 1- \frac{\| \vec x_j - \vec x_1\|}{d(1,j)} \right|_{+}  \frac{1}{d(1,j)} \\
    &\le \sum_{j \in T_3} \frac{1}{d(1,j)} \, \le \, |T_3|,
\end{align*}
where $|\cdot|_{+}:= \max\{\cdot,0\}$. Combining these estimates, we have
$$|T_3| - \frac{ \lceil n/2 \rceil (\alpha-D) }{D^2} \ge 0,$$
or equivalently,
$$\alpha  \le D +\frac{|T_3| D^2 }{\lceil n/2 \rceil}.$$

By the estimate in the proof of Lemma \ref{lm:conc}, $|T_3| \le n C^{-\sqrt{2}^k}$ for any $k \in \mathbb{N}$ and $C>0$ satisfying $(C+k)D \le \alpha -D$. Taking $C = 2$, and 
$k = \lfloor \alpha/D - 3 \rfloor> 2D \log_2 \log_2 (2D)$, we have
$$\alpha  \le D + 2 D^2 2^{-\sqrt{2}^k} \le D + \frac{2D^2}{2D} = 2D,$$
a contradiction. Therefore, $\alpha$ is at most $4D + 2D \log_2 \log_2 2D$. Repeating this argument with indices reversed completes the proof.

\end{proof}

\section{Algorithmic Lower Bounds: A Proof of Theorem \ref{thm:hardness-intro}}\label{app:hardness}

In this section, we provide a proof of Theorem \ref{thm:hardness-intro}. In particular, we aim to show that

\begin{theorem}[Restatement of Theorem \ref{thm:hardness-intro}]
The following gap version of {\bf STRESS MINIMIZATION} in dimension $r = 1$ is NP-hard: given an input graph $G$ with $n$ vertices and $L \in \mathbb{N}$, return TRUE if there exists $\vx$ such that $E(\vx) \le L$ and return FALSE if for every $\vx$, $E(\vx) \ge L + 1$. Furthermore, the problem is hard even when restricted to input graphs with diameter bounded by an absolute constant.
\end{theorem}

Our proof is based on a reduction from a version of Max All-Equal 3SAT.
The Max All-Equal 3SAT decision problem asks whether, given variables $t_1,\dots,t_\ell$, clauses $C_1,\dots,C_m \subset \{t_1,\dots,t_\ell,\bar{t}_1,\dots,\bar{t}_\ell\}$ each consisting of at most three literals (variables or their negation), and some value $L$, there exists an assignment of variables such that at least $L$ clauses have all literals equal. The Max All-Equal 3SAT decision problem is known to be APX-hard, as it does not satisfy the conditions of the Max CSP classification theorem for a polynomial time optimizable Max CSP  \cite{khanna2001approximability} (setting all variables true or all variables false does not satisfy all clauses, and all clauses cannot be written in disjunctive normal form as two terms, one with all unnegated variables and one with all negated variables).

However, we require a much more restrictive version of this problem. In particular, we require a version in which all clauses have exactly three literals, no literal appears in a clause more than once, the number of copies of a clause is equal to the number of copies of its complement (defined as the negation of all its elements), and each literal appears in exactly $k$ clauses. We refer to this restricted version as Balanced Max All-Equal EU3SAT-E$k$, and will verify that this is indeed APX-hard at the end of this appendix even for $k=6$ (Lemma~\ref{thm:MaxSAT-hardness}).

\subsubsection*{Reduction.}
Suppose we have an instance of Balanced Max All-Equal EU3SAT-E$k$ with variables $t_1,\dots,t_\ell$, and clauses $C_1,\dots,C_{2m}$. Let $\mathcal{L} = \{t_1,\dots,t_\ell,\bar{t}_1,\dots,\bar{t}_\ell\}$ be the set of literals and $\mathcal{C} =\{C_1,\dots,C_{2m}\}$ be the multiset of clauses. Consider the graph $G = (V,E)$, with
\begin{align*}
    V &=  \{v^i\}_{i\in [N_v]} \cup \{ t^i \}_{\substack{t \in \mathcal{L}\\ i \in [N_t]}} \cup \{ C^i \}_{\substack{C \in \mathcal{C}\\ i \in [N_C]}}, \\ 
    E &= V^{(2)} \setminus \bigg[ \{(t^i,\bar{t}^j)\}_{\substack{t\in \mathcal{L}\\ i,j \in [N_t]}} \cup \{(t^i,C^j)\}_{\substack{t \in C, C \in \mathcal{C}\\i \in [N_t],j \in [N_c]}} \bigg],
\end{align*}
where $V^{(2)}:=\{U\subset V \, | \, |U|=2\}$, $N_v = N_c^{20}$, $N_t = N_c^2$, $N_c =  2000 \ell^{20} m^{20}$.
In words, $G$ consists of a central clique of size $N_v$,
a clique of size $N_t$ for each literal, and
a clique of size $N_c$ for each clause;
the central clique is connected to all other cliques (via bicliques);
the clause cliques are connected to each other; each literal clique is connected to all other literals, save for the clique corresponding to its negation; and each literal clique is connected to the cliques of all clauses
it does not appear in. Note that the distance between any two nodes in this graph is at most two.

Our goal is to show that, for some fixed $L= L(\ell,m,n,p) \in \mathbb{N}$, if our instance of Balanced Max All-Equal EU3SAT-E$k$ can satisfy $2p$ clauses, then there exists a layout $\vx$ with $E(\vx) \le L$, and if it can satisfy at most $2p-2$ clauses, then all layouts $\vx'$ have objective value at least $E(\vx') \ge L + 1$. 

To this end, we propose a layout and calculate its objective function up to lower-order terms. Later we will show that, for any layout of an instance that satisfies at most $2p-2$ clauses, the objective value is strictly higher than our proposed construction for the previous case. The anticipated difference in objective value is of order $\Theta(N_t N_c)$, and so we will attempt to correctly place each literal vertex up to accuracy $o\big(N_c/(\ell n)\big)$ and each clause vertex up to accuracy $o\big(N_t/(m n)\big)$.

\subsubsection*{Structure of optimal layout.}
Let $\vx$ be a globally optimal layout of $G$, and let us label the vertices of $G$ based on their value in $\vx$, i.e., such that $\vx_1 \le \cdots \le \vx_n$. In addition, we define 
$$\hat n := n - N_v = 2 \ell N_t + 2 m N_c.$$
Without loss of generality, we assume that $\sum_{i} \vx_i = 0$. We consider the first-order conditions for an arbitrary vertex $i$. Let $S := \{ (i,j) \, | \, i<j, \, d(i,j) = 2\}$, $S_i := \{ j \in [n] \, | \, d(i,j) = 2 \}$, and 
\begin{align*}
    S_i^< &:=\{ j < i \, | \, d(i,j) = 2 \}, \qquad \qquad \qquad \qquad \, S_i^> :=\{ j > i \, | \, d(i,j) = 2 \},  \\
    S_i^{<,+} &:=\{ j < i \, | \, d(i,j) = 2,\vx_{j} \ge 0  \},\qquad \qquad S_i^{>,+} := \{ j > i \, | \, d(i,j) = 2, \vx_{j} \ge 0 \}, \\
    S_i^{<,-} &:=\{ j < i \, | \, d(i,j) = 2,\vx_{j} < 0  \},\qquad \qquad
    S_i^{>,-} := \{ j > i \, | \, d(i,j) = 2, \vx_{j} < 0 \}.
\end{align*}

We have
\begin{align*}
    \frac{\partial E}{\partial \vx_i} &= 2 \sum_{j<i} (\vx_i-\vx_j-1) - 2 \sum_{j>i} (\vx_j-\vx_i-1)  - 2 \sum_{j \in S_i^<} (\vx_i-\vx_j-1) \\
    &\qquad + 2 \sum_{j \in S_i^> } (\vx_j-\vx_i-1)  + \frac{1}{2}\sum_{j \in S_i^< } (\vx_i-\vx_j-2) - \frac{1}{2}\sum_{j \in S_i^>} (\vx_j-\vx_i-2)\\
    &= \big(2n -\tfrac{3}{2} |S_i|\big) \vx_i +2(n+1-2i)  +\frac{1}{2} \left[ \sum_{j \in S_i^<} (3\vx_j + 2)+ \sum_{j \in S_i^> } (3\vx_j-2)\right] \\
    &= \big(2n -\tfrac{3}{2} |S_i|\big) \vx_i +2(n+1-2i)+\tfrac{1}{2} \big(|S_i^{>,+}| - 5 |S_i^{>,-}| + 5 |S_i^{<,+}|-|S_i^{<,-}| \big) + \tfrac{3}{2} C_i = 0,
\end{align*} 
where
\begin{align*}
    C_i&:= \sum_{j \in S_i^{<,-} \cup S_i^{>,-}} (\vx_j +1)+ \sum_{j \in S_i^{<,+} \cup S_i^{>,+}} (\vx_j-1).
\end{align*}
The $i^{th}$ vertex has location
\begin{align} \label{eqn:xi}
    \vx_i &= \frac{2i - (n+1)}{n - \tfrac{3}{4}|S_i|} - \frac{|S_i^{>,+}| - 5 |S_i^{>,-}| + 5 |S_i^{<,+}|-|S_i^{<,-}| }{4n - 3|S_i|} - \frac{3 C_i}{4n - 3|S_i|}.
\end{align}

By Lemma \ref{lm:diam2},
\begin{align*}
    |C_i|  &\le |S_i| \, \max_{j \in [N]} \max\{ |\vx_j-1|,|\vx_j+1|\} \le 75 N_t,
\end{align*} 
which implies that $|\vx_i| \le 1 + 1/N_t^5$ for all $i \in [N]$. 


\subsubsection*{Good layout for positive instances.}
Next, we will formally describe a layout that we will then show to be nearly optimal (up to lower-order terms). Before giving the formal description, we describe the layout's structure at a high level first. Given some assignment of variables that corresponds to $2p$ clauses being satisfied, for each literal $t$, we place the clique $\{ t^i \}_{i \in [N_t]}$ at roughly $+1$ if the corresponding literal $t$ is true; otherwise we place the clique at roughly $-1$. For each clause $C$, we place the corresponding clique $ \{ C^i \}_{i \in [N_C]}$ at roughly $+1$ if the majority of its literals are near $-1$; otherwise we place it at roughly $-1$. The anchor clique $\{v^i\}_{i\in [N_v]}$ lies in the middle of the interval $[-1,+1]$, separating the literal and clause cliques on both sides.

The order of the vertices, from most negative to most positive, consists of
$$T_1, \; T_2, \; T_3^0, \,\dots, \, T_3^k, \; T_4, \; T_5^k, \, \dots, \, T_5^0, \; T_6, \; T_7,$$
where
\begin{align*}
    T_1 &= \{ \text{ clause cliques near $-1$ with all corresponding literal cliques near $+1$ } \}, \\
    T_2 &= \{ \text{ clause cliques near $-1$ with two corresponding literal cliques near $+1$} \}, \\
    T_3^\phi &= \{ \text{ literal cliques near $-1$ with exactly $\phi$ corresponding clauses near $-1$ } \}, \; \phi = 0,\dots,k, \\
    T_4 &= \{ \text{ the anchor clique } \}, \\
    T_5^\phi &= \{ \text{ literal cliques near $+1$ with exactly $\phi$ corresponding clauses near $+1$ } \}, \; \phi = 0,\dots,k, \\
    T_6 &= \{ \text{ clause cliques near $+1$ with two corresponding literal cliques near $-1$} \}, \\
    T_7 &= \{ \text{ clause cliques near $+1$ with all corresponding literal cliques near $-1$ } \},
\end{align*}
$T_3:= \bigcup_{\phi=0}^k T_3^{\phi}$, $T_6:= \bigcup_{\phi=0}^k T_6^{\phi}$, $T_c:= T_1 \cup T_2 \cup T_6 \cup T_7$, and $T_t := T_3 \cup T_5$. Let us define $\vy_i := \frac{2i - (n+1)}{n}$, i.e., the optimal layout of a clique $K_n$ in one dimension. We can write our proposed optimal layout as a perturbation of $\vy_i$.

Using the above Equation~\ref{eqn:xi} for $\vx_i$, we obtain $\vx_i = \vy_i$ for $i \in T_4$, and, by ignoring both the contribution of $C_i$ and $o(1/n)$ terms, we obtain
\begin{align*}
    \vx_i &= \vy_i - \frac{3N_t}{n}, \qquad \qquad \qquad \; i \in T_1, \qquad \qquad \qquad \vx_i = \vy_i + \frac{3N_t}{n}, \qquad \qquad \qquad \; i \in T_7, \\
    \vx_i &= \vy_i - \frac{3 N_t}{2n}, \qquad \qquad \qquad \; i \in T_2, \qquad \qquad \qquad \vx_i = \vy_i + \frac{3 N_t}{2n}, \qquad \qquad \qquad \; i \in T_6, \\
    \vx_i &= \vy_i - \frac{N_t +(k-\phi/2)N_c}{n}, \; i \in T_3, \qquad \qquad \qquad \vx_i =  \vy_i + \frac{N_t +(k-\phi/2)N_c}{n}, \;  i \in T_5.
\end{align*}

Next, we upper bound the objective value of $\vx$. We proceed in steps, estimating different components up to $o(N_c N_t)$. We recall the useful formulas
\begin{equation}\label{eqn:f1}
    \sum_{i=1}^{r-1} \sum_{j=i+1}^r \left( \frac{2(j-i)}{n} - 1 \right)^2  = \frac{(r-1)r\big(3n^2-4n(r+1)+2r(r+1)\big)}{6n^2},
\end{equation}
and 
\begin{eqnarray}\label{eqn:f2}
    \sum_{i=1}^r \sum_{j=1}^s \left( \frac{2(j-i)+q}{n}-1\right)^2 
    &=& \frac{r s}{3 n^2} \big(3 n^2 + 3q^2 + 4 r^2 + 4 s^2 - 6 n q + 6 nr \\
    &&\qquad \qquad - 6 n s - 6 q r + 6 qs - 6 rs  -2 \big) \nonumber \\
    &\le& \frac{r s^3}{3n^2} + \frac{13rs}{3n^2} \, \max \{(n-s)^2,(r-q)^2 ,r^2 \}. \nonumber
\end{eqnarray}

In what follows, we make use of the notation for different terms of the objective, as defined at the start of Appendix \ref{app:struct}. For $i \in T_c$, $|1-|\vx_i||\le 4 N_t / n$, and so
\begin{align*}
    E_{T_c}(\vx) &\le |T_c|^2 \max_{i,j \in T_c} (| \vx_i -\vx_j| - 1)^2 \le 8 m^2 N^2_c.
\end{align*} 
In addition, for $i \in T_t$, $|1-|\vx_i||\le 3 \ell N_t / n$, and so
\begin{align*}
    E_{T_t}(\vx) &\le |\{ i,j \in T_t \, | \, d(i,j) = 1 \} | \left(1 + \frac{6 \ell N_t}{n}\right)^2 + |\{ i,j \in T_t \, | \, d(i,j) = 2 \} | \; \frac{1}{4} \left( \frac{6 \ell N_t}{n}\right)^2 \\
    &\le (2\ell-1)\ell N_t^2 + 40 \frac{\ell^3 N_t^3}{n}
\end{align*}
and
\begin{align*}
    E_{T_c, T_t}(\vx) &\le \left(1 + \frac{6\ell N_t}{N}\right)^2  |\{ i \in T_c, j \in T_t \, | \, d(i,j) = 1\}| \\
    &\qquad + 2 \times \frac{1}{4} \left(2 + \frac{6\ell N_t}{n}\right)^2|\{ i \in T_2 , j \in T_3 \, | \, d(i,j) = 2\}| \\
    &\qquad + 2 \times \frac{1}{4} \left( \frac{6\ell N_t}{n}\right)^2 |\{ i \in T_1 \cup T_2 , j \in T_5 \, | \, d(i,j) = 2\}|  \\
    &\le \left(1 + \frac{18 \ell N_t}{n}\right)(2\ell  -3)N_t \, 2 m N_c  + \frac{1}{2}\left(4 + \frac{30\ell N_t}{n}\right)(m-p)N_c N_t  + \frac{1}{2}  \frac{6\ell N_t}{n}  (2m+p)N_c N_t \\
    &\le (4\ell -6) m N_t N_c + 2(m-p) N_t N_c + 100 \frac{m \ell^2 N_t^2 N_c}{n}.
\end{align*}

The quantity $E_{T_4}(\vx)$ is given by (\ref{eqn:f1}) with $r = N_v$, and so
\begin{align*}
    E_{T_4}(\vx) &= \frac{N_v(N_v -1)(3n^2 - 4 n(N_v + 1) + 2 N_v (N_v + 1))}{6 n^2} \le \frac{(n-1)(n-2) -2 \hat n  n + 3 \hat n^2 + 3 \hat n}{6}.
\end{align*}
The quantity $E_{T_1,T_4}(\vx)$ is given by (\ref{eqn:f2}) with 
$$q = 3 N_t + \hat n /2, \qquad r = p N_c, \qquad \text{and} \qquad s = N_v,$$
and so
\begin{align*}
    E_{T_1,T_4}(\vx) &\le \frac{p N_c N_v^3}{3 n^2} + \frac{13 p N_c N_v \hat n^2}{3n^2} \le \frac{p N_c }{3}(n - 3 \hat n) + \frac{16 p N_c \hat n^2}{3n}.
    \end{align*}
    

Similarly, the quantity $E_{T_2,T_4}(\vx)$ is given by (\ref{eqn:f2}) with 
$$q = \tfrac{3}{2} N_t + \hat n /2 - p N_c, \qquad r = (m-p) N_c, \qquad \text{and} \qquad s = N_v,$$
and so
\begin{align*}
    E_{T_1,T_4}(\vx) &\le \frac{(m-p) N_c N_v^3}{3 n^2} + \frac{13 (m-p) N_c N_v \hat n^2}{3n^2} \le \frac{(m-p) N_c }{3}(n - 3 \hat n) + \frac{16 (m-p) N_c \hat n^2}{3n}.
    \end{align*}

Next, we estimate $E_{T^\phi_3,T_4}(\vx)$. Using formula (\ref{eqn:f2}) with 
$$q = N_t + (k-\phi/2) N_c + \sum_{j=\phi}^k \ell_j, \qquad r = \ell_\phi, \qquad \text{and} \qquad s = N_v,$$
where $\ell_\phi := |T^\phi_3|$, we have
\begin{align*}
    E_{T_3^\phi,T_4} &\le  \frac{ \ell_\phi N^3_v}{3 n^2} + \frac{13 \ell_{\phi} N_v \hat n^2}{3 n^2} \le \frac{\ell_\phi}{3}(n- 3 \hat n) + \frac{16 \ell_\phi \hat n^2}{3 n}.
    \end{align*}

Combining the individual estimates for each $T_3^\phi$ gives us
$$E_{T_3,T_4} \le \frac{\ell N_t}{3} (n - 3 \hat n) + \frac{16 \ell N_t \hat n^2}{3n}.$$

Finally, we can combine all of our above estimates to obtain an upper bound on $E(\vx)$. We have
\begin{align*}
    E(\vx) &\le (2 \ell-1)\ell N_t^2 + (4\ell-6)m N_t N_c + 2(m-p)N_t N_c + \frac{(n-1)(n-2)}{6} - \tfrac{1}{3}n \hat n \\
    &\qquad + \tfrac{1}{2} \hat n^2 + \tfrac{2}{3}m N_c (n-3 \hat n) + \tfrac{2}{3} \ell N_t (n - 3 \hat n) + 200 m^2 N_c^2 \\
    &= \frac{(n-1)(n-2)}{6} - \tfrac{1}{2} \hat n^2 + (2 \ell-1)\ell N_t^2 + \big[(4\ell-6)m  +2(m-p)\big] N_t N_c  + 200 m^2 N_c^2\\
    &\le  \frac{(n-1)(n-2)}{6} - \ell N_t^2 - 2 (2m+p) N_t N_c + 200 m^2 N_c^2.
\end{align*}

We define the ceiling of this final upper bound to be the quantity $L$. The remainder of the proof consists of showing that if our given instance satisfies at most $2p-2$ clauses, then any layout has objective value at least $L + 1$.

Suppose, to the contrary, that there exists some layout $\vx'$ (shifted so that $\sum_{i} \vx'_i = 0$), with $E(\vx')< L+1$. 

From the analysis above, $|\vx'_i| \le 1 + 1/N_t^5$ for all $i$. Intuitively, an optimal layout should have a large fraction of the vertices at distance two on opposite sides. To make this intuition precise, we first note that

\begin{lemma}\label{lm:clique}
Let $\vx \in \mathbb{R}^n$ be a layout of the clique $K_n$. Then $E(\vx) \ge (n-1)(n-2)/6$.
\end{lemma}

\begin{proof}
The first order conditions (\ref{eqn:xi}) imply that the optimal layout (up to translation and vertex reordering) is given by $\vx'_i = \big(2i-(n+1)\big)/n$. By (\ref{eqn:f1}), $E(\vx') = (n-1)(n-2)/6$.
\end{proof}

Using Lemma \ref{lm:clique}, we can lower bound $E(\vx')$ by
\begin{align*}
    E(\vx') &= \sum_{i<j} \big( \vx'_j - \vx'_i - 1 \big)^2 -  \sum_{(i,j) \in S} \big( \vx'_j - \vx'_i - 1 \big)^2 + \frac{1}{4} \sum_{(i,j) \in S} \big( \vx'_j - \vx'_i - 2 \big)^2 \\
    &\ge \frac{(n-1)(n-2)}{6} - \sum_{(i,j) \in S} \big[ \tfrac{3}{4}(\vx'_j - \vx'_i)^2 - (\vx'_j - \vx'_i) \big].
\end{align*}

Therefore, by assumption,
\begin{align*}
    \sum_{(i,j) \in S} \big[ \tfrac{3}{4}(\vx'_j - \vx'_i)^2 - (\vx'_j - \vx'_i) \big] &\ge \frac{(n-1)(n-2)}{6} - (L+1) \\
    &\ge \ell N_t^2 + 2 (2m+p) N_t N_c - 200 m^2 N_c^2 -2.
\end{align*}

We note that the function $\tfrac{3}{4} x^2 - x$ equals one at $x=2$ and is negative for $x \in \big[0,\tfrac{4}{3}\big)$. Because $|\vx'_i| \le 1 + 1/N_t^5$ for all $i$,
$$\max_{(i,j) \in S} \big[ \tfrac{3}{4}(\vx'_j - \vx'_i)^2 - (\vx'_j - \vx'_i) \big] \le \frac{3}{4} \bigg(2 + \frac{2}{N_t^5} \bigg)^2 - \bigg(2 + \frac{2}{N_t^5} \bigg) \le 1 + \frac{7}{N_t^5}.$$
Let
$$ T' := \{ (i,j) \in S \, | \,\vx'_i \le - \tfrac{1}{6} \text{ and } \vx'_j \ge \tfrac{1}{6}  \}.$$
By assumption, $|T'|$ is at most $\ell N_t^2 + 2(2m+p-1)N_t N_c$, otherwise the corresponding instance could satisfy at least $2p$ clauses, a contradiction.

However, the quantity $\big[ \tfrac{3}{4}(\vx'_j - \vx'_i)^2 - (\vx'_j - \vx'_i) \big]$ is negative for all $(i,j) \in S \backslash T'$. Therefore,
$$ \bigg(1 + \frac{7}{N_t^5}\bigg)|T'| \ge \ell N_t^2 + 2 (2m+p) N_t N_c - 200 m^2 N_c^2 -2,$$
which implies that
\begin{align*}
|T'| &\ge \bigg(1 - \frac{7}{N_t^5}\bigg)\big(\ell N_t^2 + 2 (2m+p) N_t N_c - 200 m^2 N_c^2 -2\big) \\
&\ge \ell N_t^2 + 2 (2m+p) N_t N_c - 200 m^2 N_c^2 -2000 \\
&>\ell N_t^2 + 2(2m+p-1)N_t N_c + N_t N_c,
\end{align*}
a contradiction. This completes the proof, with a gap of one.

\subsubsection*{Balanced Max All-Equal EU3SAT-E$k$.}
We now show that the restricted form of Max All-Equal 3SAT remains APX-hard. Recall this version requires that all clauses have exactly three literals (E3SAT), no variable appears in a clause more than once (U3SAT), each variable appears in exactly $k$ clauses (SAT-E$k$), and the balanced property that the number of copies of a clause is equal to the number of copies of its complement (defined as the negation of all its elements). (Other than the new balanced property, this nomenclature follows that of \cite{Filho-thesis}.)
We will first deal with the requirements on variable appearances in clauses, and then at the end resolve the constraints of balancing clauses with their complements.

\begin{lemma}
Max All-Equal EU3SAT-E3 is APX-hard.
\end{lemma}
\begin{proof}
First we show how to reduce the number of occurrences of each variable to
be bounded above by a constant.
Theorem~1(b) of~\cite{Papadimitriou-Yannakakis-1991} shows how to construct,
for any $d$, a bounded-degree graph $G_d$ with $d$ ``distinguished'' vertices
$D$ and $O(d)$ additional vertices such that, for any cut $(S_1,S_2)$,
the number of edges across the cut is at least
$\min\{|S_1 \cap D|, |S_2 \cap D|\}$.
Now, given a variable $t$ occurring $d$ times in Max All-Equal 3SAT,
we replace its occurrences with distinct copies $t_1, \dots, t_d$
corresponding to the distinguished vertices of $G_d$, and
$O(d)$ additional variables $t_{d+1}, \dots, t_{d+O(d)}$
corresponding to the undistinguished vertices of~$G_d$.
For each edge $(t_i, t_j)$ of $G_d$, we add an equality gadget consisting of
a single all-equals clause $\{t_i, t_j\}$, encouraging those copies to be equal.
Now suppose a solution to this Max All-Equal 3SAT instance does not set all
$t_i$s to the same value.
Then we obtain a cut $(S_1,S_2)$ where one side corresponds to true $t_i$s
and the other side corresponds to false $t_i$s.
If we set the smaller set, say~$S_2$, to the majority value corresponding
to~$S_1$, then we newly satisfy the equality-gadget clauses
corresponding to the edges of the cut, and we newly unsatisfy
$|S_2 \cap D|$ original clauses that used these variables (once each).
By the cut property of $G_d$, this replacement only increases the number of
satisfied clauses.
Therefore optimal solution values remain the same under this transformation,
so we can apply it to all variables. 

Next we show how to make the number of occurrences of each variable exactly~$3$,
and no variable appears more than once in a clause,
following the proof of Proposition 2.1.2 of \cite{Feige-1998}.
For each variable $t$ occurring in $d$ clauses, replace its occurrences with
distinct copies $t_1, \dots, t_d$, and add a ring of all-equal clauses
$(t_i, t_{i+1})$ (with indices modulo~$d$).
Each variable now occurs in exactly three clauses, and no variable appears
more than once in the same clause.
If a solution violates $b$ original clauses, then its corresponding assignment
in the modified instance (with all copies set equal) still violates only $b$
clauses.
Conversely, any solution to the modified instance that violates $b$ clauses
can be converted into a solution to the original instance that violates
$O(b)$ clauses: for each variable $t$, assign it to the majority of the
$t_i$ assignments.  If the $t_i$ assignments were not all equal, then there
was at least one violated clause within the $t_i$ cycle, so the re-assignment
of the $\leq d/2$ nonmajority copies and the resulting violation of $O(d)$
original clauses can be charged to that violation.
Therefore the number of violations is preserved up to constant factors,
making this an L-reduction.

If a clause has fewer than $3$ literals, we create a new variable which will be free to take on any value. However, we now have a parity issue, since every variable must be in exactly $k$ clauses and each clause has exactly $3$ variables. Thus we triplicate each original clause (adding two additional copies) allowing there to be exactly three instances of each new free variable. This changes the number of violations by exactly a factor of $3$.

So far this construction has generated a significant number of new clauses, so we need to examine more carefully how it changes the objective to preserve APX-hardness. If our original instance had $l$ variables, $m$ clauses, and an optimal solution of $p$ clauses satisfied then the reduction produces the following changes:
\begin{itemize}
\item Add $3m$ clauses when tripling.
\item For each variable in $i>d$ clauses, we generate  the number$i+O(d)$ variables and clauses to reduce the number of instances to at most $d$ where $d$ is a constant.
\item For each variable appearing in $d>j>3$ clauses (but less than $d$), we generate $j$ new variables and $j$ new clauses
\item Variables cannot appear in less than $3$ clauses, since we tripled all of the clauses.
\item The above steps have replaced every instance of a variable with a new variable, so variables are never duplicated in any clause.
\item Each missing literal from a clause generates a new variable.
\end{itemize}
The new instance has $O(n+m)$ clauses, $O(n+m)$ variables, and the number of violated clauses remains within a constant factor of the original problem.
\end{proof}

To get the balanced property, we will make use of more than $3$ instances of each variable. we take every clause and add its complement. This transformation ensures that the number of copies of a clause is exactly equal to its complement. Further, this maintains the prior solution since the complementary clause will be satisfied if and only if the original clause was satisfied. This doubles the number of clauses in which each variable appears giving the following theorem.

\begin{lemma}
\label{thm:MaxSAT-hardness}
Balanced Max All-Equal EU3SAT-E6 is APX-hard.
\end{lemma}

\section{Approximating Vertex-Weighted Dense Max-CSPs}\label{apdx:vertex-weighted-csp}
In this section, we describe an algorithm for the Dense Max-CSP where vertices are weighted, i.e. the objective we seek to minimize is of the form
\[ \E_{i,j \sim \mu} f_{ij}(x_i,x_j) \]
where $i,j$ are sampled iid from an arbitrary measure $\mu$ over $[n]$. The usual setting in the CSP literature is where $\mu$ is the uniform measure over $[n]$, but we confirm in this Appendix that this is not essential.
\begin{theorem}[Restatement of Theorem~\ref{thm:dense-csp-weighted}]
Suppose that $\Sigma$ is a finite alphabet, $n \ge 1$ is a positive integer, and for every $i,j \in {n \choose 2}$ we have a function $f_{ij} : \Sigma \times \Sigma \to [-M,M]$. Then for any $\epsilon > 0$, there exists an algorithm which runs in time $n^{O(\log |\Sigma|/\epsilon^2)}$ and returns $x_1,\ldots,x_n \in \Sigma$ such that 
\[ \mathbb{E}\left[\mathbb{E}_{i,j \sim \mu} f_{ij}(x_i,x_j)\right] \ge \mathbb{E}_{i,j \sim \mu} f_{ij}(x^*_i,x^*_j) - \epsilon M  \]
for any $x^*_1,\ldots,x^*_n \in \Sigma$, where the outer $\mathbb{E}$ denotes expectation over the randomness of the algorithm. 
\end{theorem}
The proof of Theorem~\ref{thm:dense-csp} from \cite{mathieu2008yet} does not immediately generalize to the vertex-weighted setting. Fortunately, other algorithms for Dense CSPs can easily be made to work in this situation. To prove Theorem~\ref{thm:dense-csp-weighted} we generalize the analysis in \cite{yoshida2014approximation}. First, we start with the following ``correlation rounding'' Lemma:
\begin{lemma}[\cite{andrea2008estimating,raghavendra2012approximating,barak2011rounding}]\label{lem:corr-rounding}
Suppose $X_1,\ldots,X_n$ are arbitrary random variables valued in $[q]$. Then there exists a set $S$ with $|S| \le \lceil 1/\epsilon \rceil$ such that
\[ \E_{i,j \sim \mu} I(X_i;X_j | X_S) \le \epsilon \]
\end{lemma}
\begin{proof}
This proof is exactly the same as in \cite{raghavendra2012approximating} when we replace the uniform measure by $\mu$; for completeness, we include the proof.
Recall from the chain rule for entropy \cite{cover1999elements} that
\[ H(X_i | X_S) - H(X_i | X_S,X_j) = I(X_i;X_j | X_S)\]
hence
\[ \E_{i,j \sim \mu} H(X_i | X_S,X_j) = \E_{i \sim \mu} H(X_i | X_S) - \E_{i,j \sim \mu} I(X_i;X_j | X_S). \]
Next we form a sequence of random sets by sampling from $\mu$: $S_0$ is the empty set and we form $S_t$ by sampling $k_t \sim \mu$ independently and setting $S_t = S_{t - 1} \cup \{k_t \}$. We see that
\begin{equation}\label{eqn:entropy-decrement}
\E H(X_i | S_{t + 1}) = \E H(X_i | S_t) - \E I(X_i;X_j | S_{t}) 
\end{equation}
where as above $i,j \sim \mu$ independent of $S_t,S_{t + 1}$. 
We know that $\E_{i \sim \mu} H(X_i) \le \log(q)$ since $X_i$ is a random variable valued in an alphabet of size $q$ \cite{cover1999elements}, and that $\E H(X_i | S_t) \ge 0$ for all $t$. Hence by applying \eqref{eqn:entropy-decrement} iteratively, for any $T \ge 1$, we must have
\[ \E I(X_i;X_j | S_t) \le \frac{\log q}{T} \] 
for some $1 \le t \le T$; otherwise the average conditional entropy would become negative. Setting $T = \lceil 1/\epsilon \rceil$ proves the result.
\end{proof}
From this Lemma, we can prove Theorem~\ref{thm:dense-csp-weighted} the same way as in \cite{yoshida2014approximation}. The key point is that the Lemma only references marginal distributions of size $O(1/\epsilon)$, hence it can be used as a rounding algorithm for the Sherali-Adams relaxation of the objective. This relaxation is given explicitly below, where the set of Sherali-Adams pseudodistributions of degree $t$ is simply the collection of jointly consistent marginal distributions of size at most $t$, over which linear functions of the pseudodistribution can be optimized by linear programming in time $n^{O(t)}$ (see \cite{yoshida2014approximation} and references within).
\begin{proof}[Proof of Theorem~\ref{thm:dense-csp-weighted}]
The proof is a straightforward adaption of the argument in \cite{yoshida2014approximation}; see there for more details. The algorithm is as follows:
\begin{enumerate}
    \item Minimize $\tilde{\E}[\mathbb{E}_{i,j \sim \mu}[ f_{ij}(X_i,X_j)]]$ over all Sherali-Adams pseudoexpectations of degree $O(1/\epsilon^2)$ using linear programming. Let $\tilde{\nu}$ denote the minimizing pseudodistribution and $\tilde{\E}$ its pseudoexpectation.
    \item For all $S$ of size $O(1/\epsilon^2)$ and $x_S \in [q]^S$:
    \begin{enumerate}
        \item Sample $x_1,\ldots,x_n$ from the product measure $\bigotimes \mu(X_i | X_S = x_S)$.
        \item Compute the objective value $\E_{i,j \sim \mu} f_{ij}(x_i,x_j)$.
    \end{enumerate}
    \item Return the minimum objective value found in the loop.
\end{enumerate}
To analyze this rounding algorithm, one observes from Pinsker's inequality and Lemma~\ref{lem:corr-rounding} that for some $S$ chosen in the loop that
\[ \E_{x_S \sim \tilde{\nu}(x_S)} \E_{i,j \sim \mu} \TV^2((\tilde{\nu}(X_i | X_S = x_S) \otimes (\tilde{\nu}(X_j | X_S = x_S), \tilde{\nu}(X_i,X_j | X_S = x_S)) \le \epsilon^2. \]
Then by taking the optimal coupling of each marginal over 2 variables with its product measure version and using $\TV(P,Q) = \frac{1}{2} \sup_{|f| \le 1} (\E_P f - \E_Q f)$ we have
\begin{align*} 
&\E_{x_S \sim \tilde{\nu}(x_S)} \E_{i,j \sim \mu} \E_{X \sim \bigotimes \tilde{\nu}(X_i | X_S = x_S)} f_{ij}(X_i,X_j) \\
&\le
\E_{i,j \sim \mu} \tilde{\E}_{X} f_{ij}(X_i,X_j) + 2M \cdot \E_{x_S \sim \tilde{\nu}(x_S)} \E_{i,j \sim \mu} \TV((\tilde{\nu}(X_i | X_S = x_S)\\
&\qquad \qquad \qquad \qquad\qquad \qquad \qquad\qquad \qquad \otimes (\tilde{\nu}(X_j | X_S = x_S), \tilde{\nu}(X_i,X_j | X_S = x_S)) \\
&\le \E_{i,j \sim \mu} \tilde{\E}_{X} f_{ij}(X_i,X_j)  + 2M \cdot \big(\E_{x_S \sim \tilde{\nu}(x_S)} \E_{i,j \sim \mu} \TV^2((\tilde{\nu}(X_i | X_S = x_S) \\
&\qquad \qquad \qquad \qquad\qquad \qquad \qquad\qquad \qquad \otimes (\tilde{\nu}(X_j | X_S = x_S), \tilde{\nu}(X_i,X_j | X_S = x_S))\big)^{1/2} \\
&\le \E_{i,j \sim \mu} \tilde{\E}_{X} f_{ij}(X_i,X_j) + 2M\epsilon
\end{align*}
where in the second step we used Jensen's inequality.
Since the first term on the rhs is a lower bound on the true objective value (since it is the objective value of a relaxation), this proves the result by adjusting the value of $\epsilon$ by a  constant factor.
\end{proof}

\section{Additional Large Diameter Examples}
See Figure~\ref{fig:dodandcubic}.

\begin{figure}
    \centering
    \subfigure[Dodecahedron graph]{\includegraphics[scale=0.45]{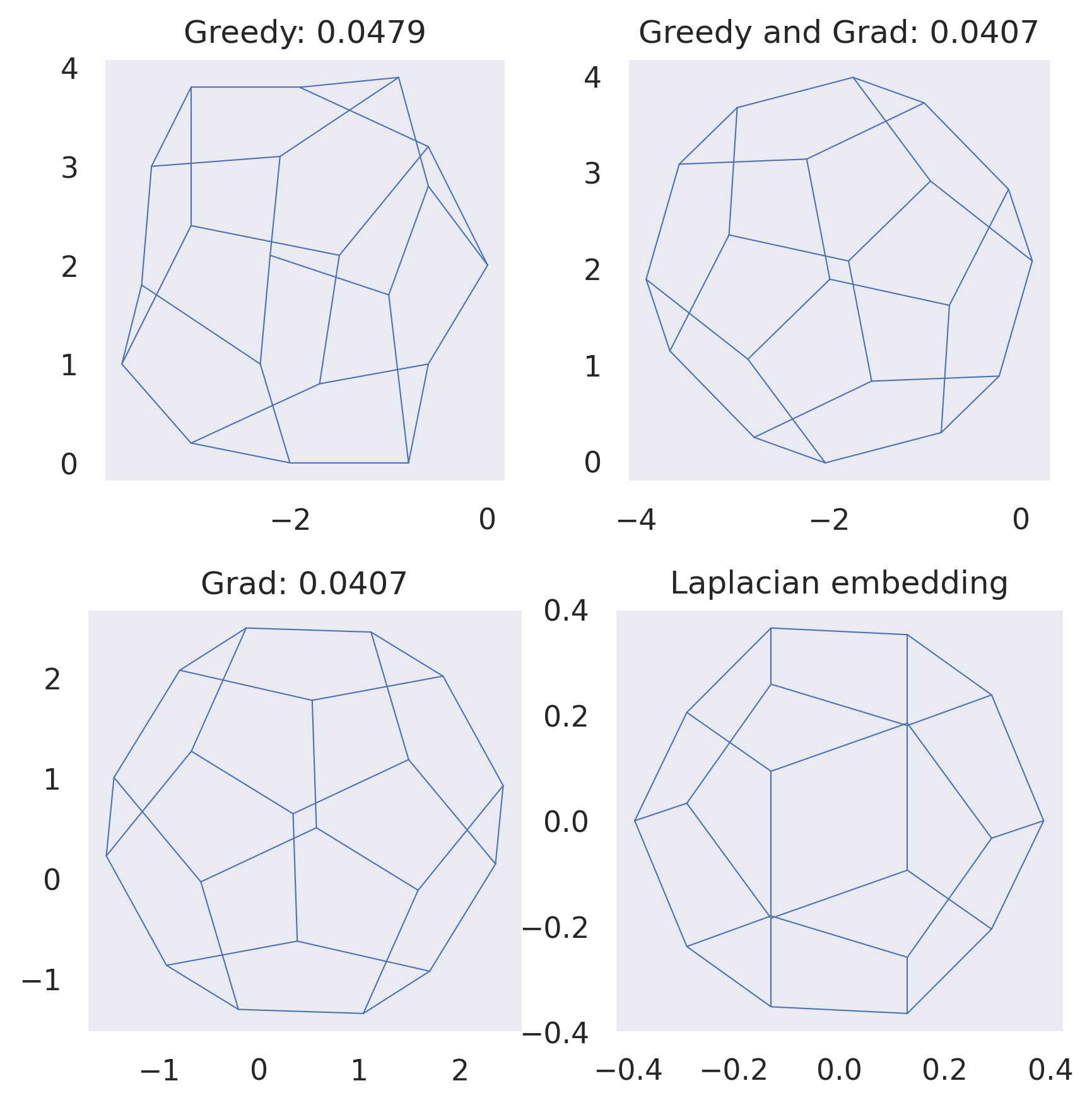}}  \qquad
    \subfigure[Cubic lattice graph]{\includegraphics[scale=0.45]{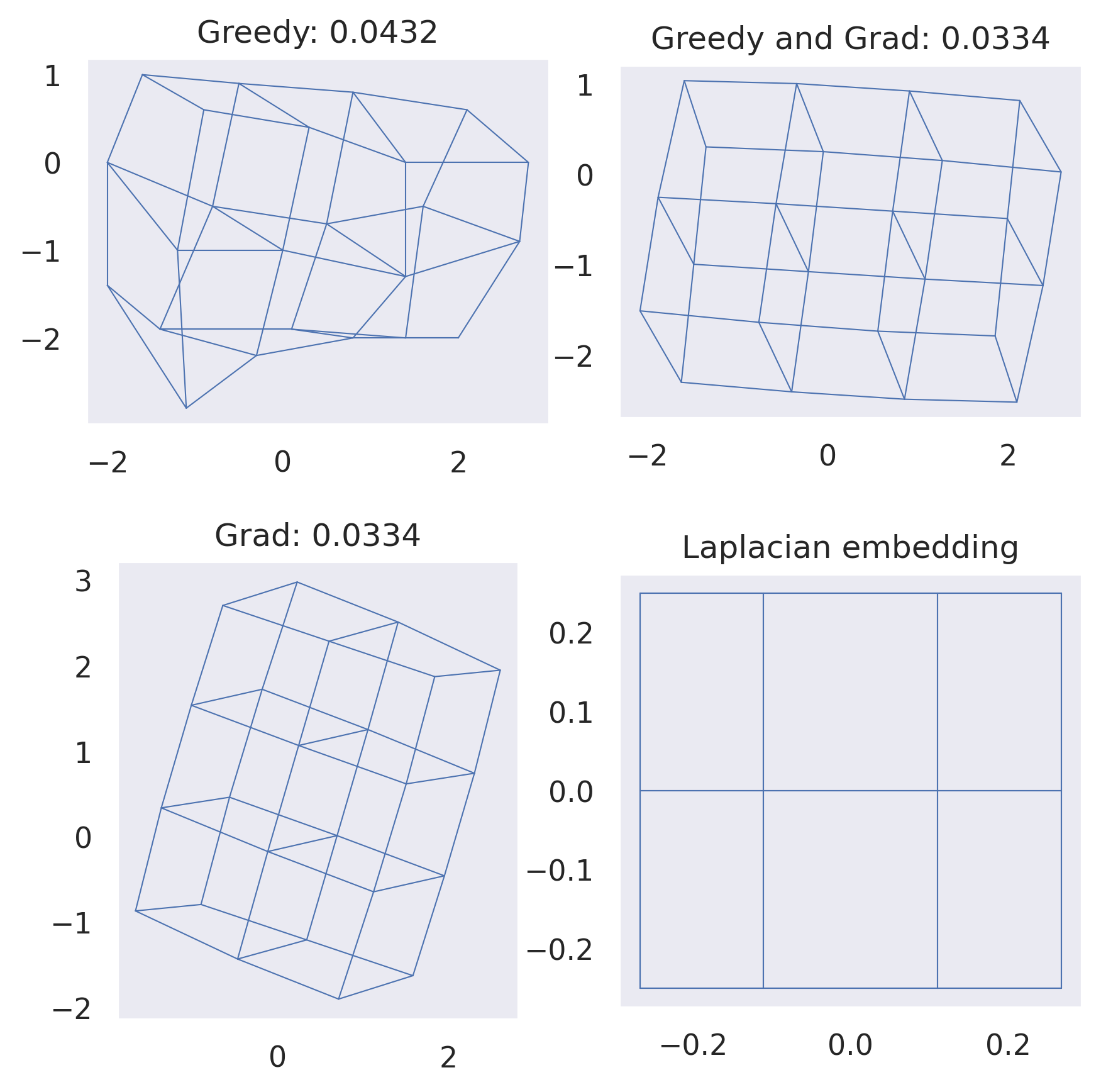}} 
    \caption{Additional Large Diameter Examples}
     \label{fig:dodandcubic}
\end{figure}

\end{document}